\documentclass[twoside,11pt]{article}
\usepackage{thm-restate}

\usepackage[preprint]{jmlr2e}

\usepackage{hyperref}       %
\usepackage{url}            %
\usepackage{booktabs}       %
\usepackage{amsfonts}       %
\usepackage{nicefrac}       %
\usepackage{microtype}      %
\usepackage{xcolor}         %
\usepackage{amsmath}
\usepackage{natbib}

\usepackage{amssymb}
\usepackage{mathtools}
\usepackage{mathrsfs}  
\usepackage{empheq}
\usepackage{bbm}
\usepackage{siunitx}
\usepackage{lastpage}

\usepackage{upgreek}

\usepackage{cleveref}
\definecolor{plum}{rgb}{0.3,0,0.7}
\usepackage{enumerate}

\usepackage{threeparttable}

\Crefname{assumption}{Assumption}{Assumptions}

\newtheorem{assumption}{Assumption}

\title{Heavy-Tailed Differential privacy}

\usepackage{dsfont}
\usepackage{xspace}

\newcommand{\eg}{e.g.\xspace}

\newcommand{\iid}{{\it i.i.d.}\xspace}

\newcommand{\ie}{{\it i.e.}\xspace}

\newcommand{\R}{\mathds{R}}
\newcommand{\er}{\widehat{\mathcal{R}}_S}
\newcommand{\erprime}{\widehat{\mathcal{R}}_{S'}}

\newcommand{\zcal}{\mathcal{Z}}

\newcommand{\Rd}{{\R^d}}
\newcommand{\levy}{L_t^\alpha}

\newcommand{\Eof}[2][]{\mathds{E}_{#1} \left[ #2 \right]}

\newcommand{\Pof}[2][]{\mathds{P}_{#1} \left( #2 \right)}

\newcommand{\prob}{\mathds{P}}
\newcommand{\normof}[1]{\left\Vert #1 \right\Vert}
\newcommand{\klb}[2]{\text{\normalfont{{KL}}}\left(#1 || #2 \right)}

\DeclareMathOperator*{\esssup}{ess\,sup}

\newcommand{\Drm}{\mathrm{D}}
\newcommand{\bregmanphi}[2]{\Drm_\Phi(#1, #2)}
\newcommand{\bregman}[3][2]{\Drm_{#1}\left(#2, #3\right)}

\newcommand{\set}[1]{\left\{ #1 \right\} }
\newcommand{\datadist}{\mu_z^{\otimes n}}
\newcommand{\Sd}{{\mathds{S}^{d-1}}}

\newcommand{\fraclap}{ \left( -\Delta \right)^{\frac{\alpha}{2}}}

\newcommand{\landau}[2][]{\mathcal{O}_{#1} \left( #2 \right)}
\newcommand{\fcal}{\mathcal{F}}
\newcommand{\acal}{\mathcal{A}}

\newcommand{\renyi}[3][\beta]{\mathrm{R}_{#1}(#2, #3)}
\newcommand{\rinfo}[3][\beta]{\mathrm{I}_{#1}(#2, #3)}
\newcommand{\calphad}{C_{\alpha,d}}
\newcommand{\kalphad}{K_{\alpha,d}}
\newcommand{\intrd}{\int_{\Rd}}
\newcommand{\intsd}{\int_{\mathds{S}^{d - 1}}}
\newcommand{\intrdzero}{\int_{\Rd\backslash\set{0}}}
\newcommand{\gammaof}[1]{\Gamma \left( #1 \right)}
\newcommand{\der}{\mathrm{d}}
\newcommand{\timeder}{\frac{\der}{ \der t}}

\newcommand{\sensitivity}{\mathcal{S}}
\newcommand{\by}[1]{\quad\text{(#1)}}
\newcommand{\parenthesis}[1]{\left( #1 \right)}
\newcommand{\tv}{\mathrm{TV}}

\newcommand{\ecal}{{\mathcal{E}}}
\newcommand{\ebeta}{\mathrm{E}_\beta}
\newcommand{\N}{{\mathds{N}}}
\newcommand{\law}[1]{{\mathrm{Law}\left( #1 \right)}}
\newcommand{\qprob}{{\mathds{Q}}}
\newcommand{\cbtwo}{{\mathcal{C}_b^2(\Rd)}}
\newcommand{\cbnorm}[2][2]{\normof{#2}_{\mathcal{C}^{#1}}}

\newcommand{\Irm}{\mathrm{I}}
\usepackage{comment}

\allowdisplaybreaks

\usepackage{lastpage}
\jmlrheading{23}{2022}{1-\pageref{LastPage}}{1/21; Revised 5/22}{9/22}{21-0000}{Dupuis et al.}

\ShortHeadings{Rényi Differential Privacy for Heavy-Tailed Dynamics}{Dupuis et al.}
\firstpageno{1}

\begin{document}

\title{Rényi Differential Privacy for Heavy-Tailed SDEs via Fractional Poincaré Inequalities}

\author{\name Benjamin Dupuis \email benjamin.dupuis@inria.fr \\
       \addr Inria, CNRS, Ecole Normale Sup\'{e}rieure\\ 
        PSL Research University, Paris, France
       \AND
       \name Mert G\"{u}rb\"{u}zbalaban \email mg1366@rutgers.edu \\
	\addr Department of Management Science and Information Systems\\
	Rutgers Business School\\
	Piscataway, NJ 08854, United States of America
       \AND
       \name Umut \c{S}im\c{s}ekli \email umut.simsekli@inria.fr \\
 \addr Inria, CNRS, Ecole Normale Sup\'{e}rieure\\ 
PSL Research University, Paris, France
\AND
 \name Jian Wang \email jianwang@fjnu.edu.cn\\
 \addr School of Mathematics and Statistics\\ 
  Key Laboratory of Analytical Mathematics and Applications (Ministry of Education) \\ Fujian Provincial Key Laboratory
of Statistics and Artificial Intelligence\\
 Fujian Normal University\\
 350007 Fuzhou, People's Republic of China
  \AND
 \name Sinan Y{\i}ld{\i}r{\i}m \email sinanyildirim@sabanciuniv.edu\\
 \addr Faculty of Engineering and Natural Sciences\\Sabanc{\i} University, Istanbul, Turkey
    \AND
	\name Lingjiong Zhu \email zhu@math.fsu.edu \\
	\addr Department of Mathematics\\
	Florida State University\\
	Tallahassee, FL 32306, United States of America
    }

\editor{}

\maketitle

\begin{abstract}
Characterizing the differential privacy (DP) of learning algorithms has become a major challenge in recent years.
In parallel, many studies suggested investigating the behavior of stochastic gradient descent (SGD) with heavy-tailed noise, both as a model for modern deep learning models and to improve their performance.
However, most DP bounds focus on light-tailed noise, where satisfactory guarantees have been obtained but the proposed techniques do not directly extend
to the heavy-tailed setting.
Recently, the first DP guarantees for heavy-tailed SGD were obtained.
These results provide $(0,\delta)$-DP guarantees without requiring gradient clipping.
Despite casting new light on the link between DP and heavy-tailed algorithms, these results have a strong dependence on the number of parameters and cannot be extended to other DP notions like the well-established \emph{Rényi differential privacy} (RDP).
In this work, we propose to address these limitations by deriving the first RDP guarantees for heavy-tailed SDEs, as well as their discretized counterparts. 
Our framework is based on new Rényi flow computations and the use of %
well-established fractional Poincaré %
inequalities.
Under the assumption that such inequalities are satisfied, we obtain DP guarantees that have a much weaker dependence on the dimension compared to prior art.
\end{abstract}


\section{Introduction}
\label{sec:introduction}

\paragraph{Setup.} Let $\zcal := \mathcal{X} \times \mathcal{Y}$ be a data space endowed with a $\sigma$-algebra $\fcal$ and a probability distribution $\mu_z$, where $\mathcal{X}$ denotes the space of input data and $\mathcal{Y}$ denotes the space of output data.
In the context of supervised learning, stochastic optimization algorithms are used to solve the following \textit{empirical risk minimization} (ERM) problem
\begin{align}\label{ERM:problem}
    \min \left\{ \er(w) := \frac1{n} \sum\nolimits_{i=1}^n \ell(w,z_i): ~ w\in \Rd\right\},
\end{align}
where $\ell : \Rd \times \zcal \to \R_+$ is a loss function and $S := (z_1, \dots, z_n) \sim \datadist$ is a dataset sampled from the data distribution. With these notations, stochastic learning algorithms for solving the ERM problem can be seen as randomized mappings $\acal : S \mapsto W_S \in \Rd$. 

Due to the rapid development of machine learning, it has become essential to certify the \textit{differential privacy} (DP) of learning algorithms. 
Informally, this notion can be understood as the question: given two \emph{neighbouring} datasets $S,S'\in \zcal^n$ (\ie, differing by only one data point), how hard is it to distinguish the probability distributions of $W_S$ and $W_{S'}$, \ie, can we identify whether one particular data point has been used during training. 
A related question then is the creation of algorithms with provable DP guarantees.
Several formalizations of DP have been proposed \citep{dwork_differential_2006,dwork_algorithmic_2013,dwork_concentrated_2016}.
In our paper, we are in particular interested in the \emph{Rényi differential privacy} (RDP) \citep{mironov_renyi_2017}, presented as a natural relaxation of classical DP conditions. 
More precisely, a learning algorithm $\acal : S \mapsto W_S$ is said to be $(\beta, \kappa)$-RDP when $\renyi{\mathrm{Law}(W_S)}{\mathrm{Law}(W_{S'})} \leq \kappa$, for neighboring datasets $S$ and $S'$, where $\renyi{\cdot}{\cdot}$ denotes the Rényi divergence for $\beta > 1$.
We provide in \Cref{sec:dp-background} a detailed presentation of these notions.

In our study, we aim to prove RDP guarantees for heavy-tailed SDEs. More precisely, given $S\in\zcal^n$ and a time horizon $T>0$, the learning algorithm is given by $\acal : S \mapsto W_T$, where $(W_t)_{t\geq 0}$ is a solution of the following stochastic differential equation (SDE) in $\Rd$
\begin{align}
    \label{eq:sde-intro}
    \der W_t = - \nabla \er (W_t) \der t + \sigma_\alpha \der \levy + \sigma_2 \sqrt{2} \der B_t,
\end{align}
where $\sigma_\alpha,\sigma_2 \geq 0$ are fixed, $(\levy)_{t \geq 0}$ is a rotationally invariant $\alpha$-stable L\'{e}vy process with tail index $\alpha \in %
(0,2)$ (see \Cref{sec:levy-sdes-background}), and $(B_t)_{t\geq 0}$ is a standard Brownian motion.
We also consider Euler-Maruyama discretizations of \Cref{eq:sde-intro}, which corresponds to heavy-tailed stochastic gradient descent (SGD), as studied by \citet{simsekli_differential_2024} in the case $\sigma_2 = 0$.

These heavy-tailed algorithms have recently gained a lot of attention in learning theory. First, it was shown in several studies, both empirically and theoretically, that SGD can produce heavy-tailed distributions under certain choices of hyperparameters \citep{simsekli_tail-index_2019,gurbuzbalaban_heavy-tail_2021}. 
Following these findings, it was shown that the presence or injection of heavy-tailed noise can ensure good generalization properties \citep{raj_algorithmic_2023,raj_algorithmic_2023-1,dupuis_generalization_2024,lim_chaotic_2022}, and improve the compressibility of the model \citep{barsbey_heavy_2021,wan_implicit_2023}. In addition to these properties, it has been shown that heavy-tailed noise can lead to satisfactory minimization of the empirical risk in convex and non-convex settings \citep{wang_convergence_2021,simsekli_tail-index_2019}, and in the context of one-hidden-layer neural networks \citep{wan_implicit_2023}. 
These observations support the interest in obtaining DP guarantees for heavy-tailed dynamics such as \Cref{eq:sde-intro}.

\begin{table}[!t]
\tabcolsep=0.11cm
\small
\centering
\begin{threeparttable}
\begin{tabular}{@{} l S[table-format=6.0] l S[table-format=5.0] cc @{}} 
\toprule
{Paper}  & {Type of DP} & {Assumptions} & {Guarantee} \\
\midrule
\citet[Thm. 1]{abadi_deep_2016} & {$(\varepsilon, \delta)$-DP} & {S.} & $\varepsilon \simeq \landau{\frac{ {b}\sqrt{T \log(1/\delta)}}{n \sigma}}$ \\
\citet[Thm. 2]{asoodeh_privacy_2023} & {$(\varepsilon, \delta)$-DP} & {S.} & {$\varepsilon \simeq \landau{\frac1{\sigma^2}}$~\text{for}~$\delta \simeq \landau{\frac{b}{n}}$} \\
\citet[Thm. 2]{chourasia_differential_2021} & {$(\beta, \kappa)$-RDP} & S., LSI. & {$\kappa \simeq \landau{\frac{\beta \sensitivity_g^2}{n^2 \sigma^2}}$}  \\
\citet[Thm 3.1]{ryffel_differential_2022} & {$(\beta, \kappa)$-RDP} & L., LSI. & {$\kappa \simeq \landau{\frac{\beta L^2}{n^2 \sigma^2}}$}  \\
\citet[Thm. 1.3]{altschuler_privacy_2022}  & {$(\beta, \kappa)$-RDP} & {C., L., $\nabla$-L.} & {$\kappa \simeq \landau{\frac{\beta L^2}{n^2 \sigma^2} \min (T, n)}$} \\
\citet[Thm. 3.3]{ye_differentially_2022}  & {$(\beta, \kappa)$-RDP} & {SC., S., $\nabla$-L.} & {$\kappa \simeq$ {$\landau{\frac{\beta \sensitivity_g^2}{\sigma^2 {b}^2}}$}} \\
\citet[Thm. 3.3]{chien_langevin_2025-1}  & {$(\beta, \kappa)$-RDP} & {L., $\nabla$-L., LSI.} & {$\kappa \simeq \landau{\frac{\beta T}{\sigma^2 n^2}}$} \\
\citet[Thm. 3.3]{chien_langevin_2025-1}  & {$(\beta, \kappa)$-RDP} & {L., $\nabla$-L., SC.} & {$\kappa \simeq \landau{\frac{\beta}{\sigma^2 n^2}}$} \\
\bottomrule
\end{tabular}
    \caption{Comparison of existing guarantees for DP-SGD with Gaussian noise. Simplified versions are presented; we refer to the papers for details. \textit{Abbreviations:} \textrm{S.} (finite sensitivity / clipping), \textrm{LSI.} (log-Sobolev inequality), $\nabla$-\textrm{L.} (gradient Lipschitz), \textrm{C.} (convex), \textrm{SC.} (strongly-convex), \textrm{L.} ($L$-Lipschitz). \textit{notations:} $T$ (number of iterations / time), $\sensitivity_g$ (gradient sensitivity), {$b \leq n$} (batch size). 
    }
\label{tab:dp-gaussian-noise}
\vspace{-5mm}
\end{threeparttable}
\end{table}

\paragraph{Related works.}
Many differentially private mechanisms have been studied with Gaussian \citep{wang_differentially_2018,feldman_privacy_2018} or Laplace noise \citep{chaudhuri_differentially_2011,kuru_differentially_2022}.
The main inspiration for our work is the well-established differentially private SGD (DP-SGD) mechanism \citep{bassily_differentially_2014}, which involves adding (Gaussian) noise in SGD iterations, along with potential projections and gradient clipping operations \citep{abadi_deep_2016,yu_differentially_2019,chen_understanding_2020}.
In our study, this model corresponds to a discretized version of \Cref{eq:sde-intro} with $\sigma_\alpha = 0$, with projection or clipping of the gradient on a ball of finite radius, which we could also include in our model.

In particular, \citet{chourasia_differential_2021} studied the following noisy SGD recursion
\begin{align}
    \label{eq:chourasia-recursions}
    W_{k+1} = \Pi_{\mathcal{C}} \left( W_k - \eta \nabla \er (W_k) + \sigma_2 \sqrt{2\eta} \xi_k \right), \quad \xi_k \overset{\iid}{\sim} \mathcal{N}(0,I_d),
\end{align}
where $\Pi_{\mathcal{C}}$ denotes the projection on a convex set. They obtained a $(\beta, \kappa)$-RDP guarantee with $\kappa = \landau{\beta \sensitivity_g^2 / n^2}$, where $\sensitivity_g$ is the gradient sensitivity, %
formally defined in \Cref{sec:setup-assumptions}. Similar RDP and DP guarantees have been derived by other authors for convex and non-convex losses, using a wide range of tools \citep{asoodeh_privacy_2023,altschuler_privacy_2022}. A large part of these works make crucial use of the \textit{logarithmic Sobolev inequality} (LSI) \citep{gross_logarithmic_1975-1} to derive RDP guarantees by exploiting the mixing properties of the underlying processes \citep{ye_differentially_2022,chien_langevin_2025-1,ganesh_faster_2020,ryffel_differential_2022}. These works identify assumptions on the loss such that the LSI can be satisfied, which they exploit to make their bounds time-uniform.
We summarize some of the aforementioned results in \Cref{tab:dp-gaussian-noise}.

Unfortunately, these works cannot be directly extended to the heavy-tailed case (\ie, when $\sigma_\alpha > 0$ in \Cref{eq:sde-intro}) for two main reasons. First, the proposed derivations make explicit use of the Gaussian structure of the noise. For example, \citet{chourasia_differential_2021} exploit the Fokker-Planck equations associated with continuous-time interpolations of \Cref{eq:chourasia-recursions}. Second, it has been noted by \citet{dupuis_generalization_2024} that the LSI that is instrumental to the Gaussian case might not be available in the heavy-tailed setting. So far, these issues have prevented the obtainment of RDP guarantees for heavy-tailed SDEs.

Recent studies derived DP guarantees under heavy-tailed noise. \citet{ito_privacy_2023} investigated the DP properties of linear dynamical systems under $\alpha$-stable noise, while \citet{zawacki_heavy-tailed_2025} proposed a general $\alpha$-stable DP mechanism, which is, however, not directly related to SGD and, therefore, independent of our study. 
In the case of SGD, \citet{asi_private_nodate} studied the DP of stochastic convex optimization under the assumption of heavy-tailed gradients, in the sense of finite moments up to order $k$. However, this study does not directly contain the case of $\alpha$-stable noise, which is the focus of our paper.
More recently, \citet{simsekli_differential_2024} investigated the following heavy-tailed noisy SGD recursion in $\Rd$
\begin{align}
    \label{eq:simsekli-dp-recursion}
    W_{k+1} = W_k - \eta \nabla \er (W_k) + \sigma_{\alpha} \eta^{1 / \alpha} \xi_k, \quad \xi_k \overset{\iid}{\sim} \mathcal{S}\alpha\mathcal{S} := \law{L_1^\alpha},
\end{align}
which is a discretization of \Cref{eq:sde-intro} with $\sigma_2 = 0$. These authors obtained $(0,\delta)$-DP guarantees (see \Cref{def:dp-epsilon-delta}) with $\delta = \landau{d^{(1 + \alpha) / 2} / n}$ and without requiring gradient clipping.
Their study is based on a Markov chain perturbation analysis and, most importantly, can handle unbounded gradients, while many works rely on bounded gradients or finite sensitivity assumptions \citep{chourasia_differential_2021}. 
Despite successfully obtaining the first DP guarantees for heavy-tailed SGD, this work has three main disadvantages: \textrm{(i)} $(0,\delta)$-DP is weaker than RDP, as we recall in \Cref{lemma:rdp-to-zero-delta-DP}, \textrm{(ii)} the proposed bound has a strong dependence on the dimension, which might render it vacuous in practice, and \textrm{(iii)} it has an intricate and not explicit dependence on the noise scale $\sigma_\alpha$, making its impact on the bound unclear.

\defcitealias{simsekli_differential_2024}{Simsekli et al.}

\begin{table}[!t]
\tabcolsep=0.11cm
\small
\centering
\begin{threeparttable}
\begin{tabular}{@{} l S[table-format=6.0] l S[table-format=5.0] cc @{}} 
\toprule
{Paper}  & {Type of DP} & {Assumptions} & {Guarantee} \\
\midrule
\citet[Thm. 10]{simsekli_differential_2024} & {$(0,\delta)$-DP} & {Ps.-$\nabla$-L., D.}& {$\delta \simeq \landau{\frac{d^{(1+\alpha) / 2}}{n}}$} \\
\textbf{Ours} (\Cref{thm:multifractal-case-finite-sensitivity}) & {$(\beta, \kappa)$-RDP} & {S., $\sigma_2>0$} & {$\kappa\simeq\landau{\frac{\beta \sensitivity_g^2 T}{n^2 \sigma_2^2}}$} \\
\textbf{Ours} (\Cref{thm:multifractal-case-finite-sensitivity}) & {$(\beta, \kappa)$-RDP} & {S., FPI, $\sigma_2>0$} & {$\kappa\simeq\landau{\frac{\beta^2 \sensitivity_g^2}{n^2 \sigma_2^2}}$} \\
\textbf{Ours} (by \Cref{lemma:rdp-to-zero-delta-DP}) & {$(0,\delta)$-DP} & {S., FPI, $\sigma_2>0$}  & {$\delta\simeq\landau{\frac{\sensitivity_g}{n\sigma_2}}$} \\
\textbf{Ours} (\Cref{thm:discrete-pure-jump,thm:dp_guarantees_sdes_pure_jump}) & {$(\beta, \kappa)$-RDP} & {S., $\sigma_2=0$}  & {$\kappa\simeq\landau{\frac{\beta d^{1 - \alpha / 2} \sensitivity_g^2  T}{n^2 \sigma_\alpha^\alpha R^{2 - \alpha} }}$} \\
\textbf{Ours} (\Cref{thm:discrete-pure-jump,thm:dp_guarantees_sdes_pure_jump}) & {$(\beta, \kappa)$-RDP} & {S., FPI, $\sigma_2=0$}  & {$\kappa\simeq\landau{\frac{\beta^2 d^{1 - \alpha / 2} \sensitivity_g^2  }{n^2 \sigma_\alpha^\alpha R^{2 - \alpha} }}$} \\
\textbf{Ours} (by \Cref{lemma:rdp-to-zero-delta-DP}) & {$(0,\delta)$-DP} & {S., FPI, $\sigma_2=0$}  & {$\delta\simeq\landau{\frac{d^{(2 - \alpha) / 4} \sensitivity_g}{n \sigma_\alpha^{\alpha / 2} R^{1 - \alpha / 2} }}$} \\
\bottomrule
\end{tabular}
\caption{Comparison of guarantees for DP-SGD with $\alpha$-stable noise. Simplified versions are presented, we refer to the papers for details. \textit{Abbreviations:} \textrm{S.} (finite sensitivity / clipping), \textrm{FPI.} (fractional Poincaré inequality), \textrm{D.} (dissipativity), \textrm{Ps.}-$\nabla$-\textrm{L.} (pseudo gradient Lipschitz). \textit{notations:} $\sensitivity_g$ (gradient sensitivity), $R$ (quantity appearing in \Cref{thm:discrete-pure-jump,thm:dp_guarantees_sdes_pure_jump}, which might depend on $(d,T,\beta)$).}
\label{tab:dp-heavy-tails}
\vspace{-5mm}
\end{threeparttable}
\end{table}

\subsection{Contributions}
\label{sec:contributions}

In our work, we propose to address the issues mentioned above by proving RDP guarantees for heavy-tailed SDEs and their discrete counterparts. Our framework is based on the extension of the entropy flow computations of \citet{chourasia_differential_2021} in the heavy-tailed setting. 
In contrast with the relative entropy flow derivations obtained by \citet{dupuis_generalization_2024} in the context of generalization bounds, we obtain \emph{Renyi divergence flows} (\ie, the time derivative of the Rényi divergence) for SDEs driven by general Lévy processes. 
In order to circumvent the lack of LSI in this heavy-tailed case, we show that we can obtain time-uniform RDP guarantees based on the fractional versions of the celebrated Poincaré inequalities \citep{wang_functional_2015,mouhot_fractional_2009}.
Based on this method, we obtained the first RDP guarantees for heavy-tailed (\ie, $\alpha$-stable) SGD. 
Our detailed contributions are summarized below and in \Cref{tab:dp-heavy-tails}.  %
\begin{itemize}
    \item We provide a general framework to compute the flow of Rényi divergences along Lévy-driven SDEs and study the associated functional inequalities. Our proof technique applies to a very general class of Lévy processes, which makes it more general than existing works on heavy-tailed SGD.
    
    \item In the multifractal case (\ie, when both $\sigma_\alpha > 0$ and $\sigma_2 > 0$), we obtain dimension-independent $(\beta, \kappa)$-RDP guarantees with $\kappa = \landau{\beta^2 / (n^2\sigma_2^2)}$, under a fractional Poincaré inequality (FPI) assumption. We also obtain a guarantee in $\kappa = \landau{\beta T / (n^2\sigma_2^2)}$ when this assumption is removed.
    
    \item In the pure-jump $\alpha$-stable case (\ie, when $\sigma_2 = 0$), we draw a new link with Bourgain-Brezis-Mironescu's type formulas \citep{bourgainBrezisMironescu2001} and show that, under slightly stronger assumptions, we can obtain $(\beta, \kappa)$-RDP guarantees with $\kappa = \landau{\beta^2 d^{1 - \alpha / 2} / (n^2 \sigma_\alpha^\alpha)}$ under FPI and $\kappa = \landau{\beta d^{1 - \alpha / 2} T / (n^2 \sigma_\alpha^\alpha)}$ without FPI.
    Our results imply $(0,\delta)$-DP guarantees which have much weaker dependence on the dimension $d$ than in existing works in both the multifractal and the pure-jump case. 
    
    \item Finally, we extend our RDP to the discrete-time setting, which is used in practical application.  Moreover, as a sanity check, we investigate the satisfiability of fractional Poincaré inequalities in this case, hence providing theoretical foundation for our main assumptions. As a by-product of our analysis, we derive new stability results for fractional Poincaré inequalities.
\end{itemize}

\paragraph{Organization of the paper.} We present some technical background on differential privacy, Lévy processes, and fractional Poincaré inequalities in \Cref{sec:preliminaries}. Our main results are discussed in \Cref{sec:dp-of-heavy-tailed-sdes}, where we detail our setup, assumptions, the Rényi flow computations, and our DP guarantees in the multifractal and the pure-jump cases. Finally, \Cref{sec:discrete-time} is dedicated to the analysis of the discrete-time algorithm and the stability properties of fractional Poincaré inequalities.
All omitted proofs are postponed to the appendix.

\paragraph{Notations.} 
For all $x \in \Rd$, we denote by $\normof{x}$ the Euclidean norm of $x$. 
For a matrix $A \in \R^{d \times d}$, the notation $\normof{A}_2$ refers to the spectral norm.
For a matrix mapping $A : \Rd \to \R^{d \times d}$, the notation $\normof{A}_\infty$ is also understood with respect to the spectral norm.
The open ball centered at $x$ with radius $r>0$ is denoted by $B_r(x)$.
For any random variable $X$, the law of $X$ is written as $\law{X}$. 
If $\mu$ is a probability measure and $T$ is a measurable map, we write $T_\# \mu$ for the pushforward measure.
We abbreviate the partial derivatives $\partial / \partial x$ by $\partial_x$.
Let $\cbtwo$ denote the set of twice continuously differentiable functions $u : \Rd \to \R$ which are bounded, along with their first- and second-order derivatives. 
For $f \in \cbtwo$, we denote $\cbnorm{f} := \max_{|\gamma| \leq 2} \normof{\partial^\gamma f}_\infty$, where $\partial^\gamma f := \partial_{x_1}^{\gamma_1} \dots \partial_{x_d}^{\gamma_d}$ and $|\gamma| := \gamma_1 + \dots + \gamma_d$ for any $\gamma \in \N^d$. 
The Laplacian operator in $\Rd$ is denoted as $\Delta := \partial_{x_1}^2 + \dots + \partial_{x_d}^2$.

\section{Preliminaries}
\label{sec:preliminaries}

In this section, we introduce some necessary technical background. We start by discussing the classical notions of differential privacy and Rényi differential privacy in \Cref{sec:dp-background}, and then present Lévy processes in \Cref{sec:levy-sdes-background}. Finally, we provide an introduction to fractional Poincaré inequalities in \Cref{sec:fractional-pi-background}.

\subsection{Background on differential privacy}
\label{sec:dp-background}
 
Let $\zcal$ be a data space endowed with a $\sigma$-algebra $\mathcal{F}$. Given two datasets $S,S' \in \zcal^n$ %
for some $n\ge0$, %
we say that they are \emph{neighbors}, and we denote $S \simeq S'$, if both datasets differ by only one element. 
In this paper, we define a learning algorithm as a mapping
$\acal : \bigcup_{n=0}^{\infty} \zcal^n \longrightarrow \mathcal{P}(\Rd)$, $S \longmapsto \acal(S)$,
where $\mathcal{P}(\Rd)$ denotes the set of Borel probability measures on $\Rd$.
We will refer to the distribution $\acal(S)$, for $S\in
\bigcup_{n=0}^{\infty} %
\zcal^n$, as %
a %
\emph{posterior distribution}.

The classical notion of \textit{Differential Privacy} (DP), as introduced in \cite{dwork_algorithmic_2013, dwork_boosting_2010}, regards the possibility of distinguishing the posterior distributions induced by two neighboring datasets.

\begin{definition}[$(\varepsilon, \delta)$-DP]
    \label{def:dp-epsilon-delta}
    Let $\varepsilon\ge0$ and $ \delta\in[0,1]$. %
    The learning algorithm $\acal$ is said to satisfy $(\varepsilon, \delta)$-DP if, for any $S \simeq S'$ and any Borel set $B$ in $\Rd$,  
        $\acal(S)(B) \leq e^\varepsilon \acal(S')(B) + \delta.$
\end{definition}

In \cite{mironov_renyi_2017}, the author defined a notion of DP based on the Rényi divergence between neighboring posterior distributions. 

\begin{definition}[$(\beta, \kappa)$-RDP]
    \label{def:rdp}
    Let $\kappa> 0$ and $\beta > 1$.
    The learning algorithm $\acal$ is said to satisfy $(\beta, \kappa)$-Rényi differential privacy (RDP) if, for any $S \simeq S'$,  
        $\renyi{\acal(S)}{\acal(S')} \leq \kappa$,
    where, given two Borel probability measures $\qprob \ll \prob$, we define their Rényi divergence by 
    \begin{align}
        \label{eq:inside-renyi}
        \renyi{\qprob}{\prob} := \frac{1}{\beta - 1} \log \ebeta(\qprob, \prob) \quad \text{with}\,\,\,\, \ebeta(\qprob, \prob) := \int \left( \frac{\der \qprob}{\der \prob}\right)^{\beta} \der \prob.
    \end{align}
\end{definition}
By convention, we set $ \renyi{\qprob}{\prob} := +\infty$ when $\qprob \not\ll \prob$. 
Note that we also define $\renyi[1]{\qprob}{\prob} := \klb{\qprob}{\prob} := \int \log(\der \qprob / \der \prob) \der \qprob$, which is the Kullback-Leibler divergence between $\qprob$ and $\prob$. 
It can be seen that $\beta \mapsto \renyi{\cdot}{\cdot}$ is an increasing and continuous map for $\beta \geq 1$ \citep{van_erven_renyi_2014}.
The particular interest of the above notion comes from the fact that RDP implies DP.
\begin{lemma}
    \label{lemma:rdp-implies-dp}
    For any $\kappa>0$, %
    $\beta > 1$, and $ \delta\in (0,1]$, $(\beta,\kappa)$-RDP implies $(\varepsilon, \delta)$-DP, with $ \varepsilon = \kappa + \frac{\log(1/\delta)}{\beta - 1}$.
\end{lemma}
The proof of this result is elementary and can be found in \cite{abadi_deep_2016,mironov_renyi_2017}. This result is also mentioned in \cite{asoodeh_three_2021}, who also prove tighter conversion results.
In the case where $\varepsilon=0$, we also have the following conversion lemma.

\begin{lemma}
    \label{lemma:rdp-to-zero-delta-DP}
    Let $\kappa> 0$ and $\beta > 1$.
    Then, $(\beta,\kappa)$-RDP implies $(0, \sqrt{\kappa/2})$-DP.
\end{lemma}

\begin{proof}
    Let $S \simeq S'$ in $ \bigcup_{n=0}^{\infty}\zcal^n$. By \cite[Theorem 3]{van_erven_renyi_2014}, the R\'{e}nyi divergence $\renyi{\qprob}{\prob}$ is non-decreasing in $\beta$, and thus,
    \begin{align*}
        \klb{\acal(S)}{\acal(S')} \leq \renyi[\beta]{\acal(S)}{\acal(S')} \leq \kappa.
    \end{align*}
    Therefore, by Pinsker's inequality \citep[Theorem 31]{van_erven_renyi_2014},
    we obtain that $\tv(\acal(S),\acal(S')) \leq \sqrt{\kappa / 2}$.
\end{proof}

We also define the following R\'{e}nyi information, which plays a crucial role in \cite{chourasia_differential_2021}, as well as in the proofs of our main results. We follow the convention of \cite{chourasia_differential_2021} for the normalization of this quantity.

\begin{definition}
    \label{def:renyi-info}
    Let $\qprob$ and $\prob$ be two probability measures on $\Omega$ such that $\qprob \ll \prob$ and $\beta > 1$. Assume that the Radon-Nikodym derivative $\der \qprob / \der \prob$ is differentiable. Then, the Rényi information of $\qprob$ with respect to $\prob$ is defined as
    \begin{align*}
        \rinfo{\qprob}{\prob} := \int \left( \frac{\der \qprob}{\der \prob}\right)^{\beta - 2} \normof{\nabla  \frac{\der \qprob}{\der \prob}}^2 \der\prob.
    \end{align*}
\end{definition}

\subsection{Lévy processes and infinitely divisible distributions}
\label{sec:levy-sdes-background}

An interesting aspect of our approach based on Rényi flows is that the proof techniques extend to a very general class of Lévy processes. This fact is highlighted in \Cref{sec:renyi-flows-general-levy-process} in the appendix. In order to present it clearly, we give below some technical background on Lévy processes.

Let $(\Omega, \fcal, \prob)$ be a fixed probability space.
A stochastic process $(L_t)_{t\geq 0}$ is called a Lévy process, if $L_0 = 0$ almost surely, and it satisfies the following properties.
\begin{itemize}
    \item \textbf{Stationary increments:} For all $s \leq t$, $\law{L_t - L_s} = \law{L_{t - s}}$.
    \item \textbf{Independent increments:} $L_t - L_s$ is independent from the $\sigma$-algebra $\sigma (L_u, ~u \leq s)$.
    \item \textbf{Stochastic continuity:} For all $\varepsilon > 0$, we have $\lim_{s\to t} \Pof{\normof{L_t - L_s} > \varepsilon} = 0$.
\end{itemize}
Under these conditions, the distribution of $L_t$ is \emph{infinitely divisible} \citep{schilling_introduction_2016}, which means that its distribution can be written as an $m$-fold convolution for all $m\in\N^\star:=\{1,2,\cdots\}$. There is a one-to-one correspondence between Lévy processes and infinitely divisible distributions.
Classically, this implies that %
the characteristic function of the L\'evy process $(L_t)_{t\ge0}$  can be expressed for all $\xi \in \Rd$ as $\Eof{\exp (i \xi^{\top} L_t )} = \exp(-t \psi(\xi))$, where $i:=\sqrt{-1}$ and $\psi$ is called the characteristic exponent, which is given by the celebrated Lévy-Khintchine formula \citep{bottcher_levy_2013}
\begin{align}
    \label{eq:levy-khintchine}
    \psi(\xi) = -i b^{\top} \xi + \frac1{2} \xi^{\top} \Sigma \xi + \intrdzero \left( 1 - e^{i\xi^{\top} z} + i \xi^{\top} z \chi (\normof{z}) \right) \der \nu (z),\quad \xi\in \R^d,
\end{align}
where $b \in \Rd$, $\Sigma \in \R^{d \times d}$ is a symmetric positive semi-definite matrix, $\chi$ satisfies\footnote{$\chi$ %
can be chosen %
arbitrary as soon as it satisfies certain properties, see \citep{bottcher_levy_2013}.} $\chi(s) := (1 + s^2)^{-1}$, and $\nu$ is a positive Radon measure on $\Rd$ such that
    $\intrdzero \min (1, \normof{z}^2) \der \nu(z) < +\infty$.
The measure $\nu$ is called the \emph{Lévy measure} and $(b, \Sigma, \nu)$ is the \emph{Lévy triplet} of $(L_t)_{t \geq 0}$.
For instance, the standard Brownian motion $(B_t)_{t \geq 0}$ in $\Rd$ is a Lévy process with triplet $(0, I_d, 0)$.
In our paper, we are in particular interested in the rotationally invariant $\alpha$-stable Lévy process $(\levy)_{t\geq 0}$, whose characteristic %
exponent is given by $\psi(\xi) = \normof{\xi}^\alpha$, with $\alpha \in (0, 2]$. When $\alpha < 2$, its Lévy triplet is given by $(0,0,\nu)$ \citep[Example 2.4.d]{bottcher_levy_2013}, with
\begin{align}
    \label{eq:calphad-levy-measure}
    \der \nu(z) := \calphad \frac{\der z}{\normof{z}^{\alpha + d}} \quad \text{and }\quad \calphad := \alpha 2^{\alpha - 1} \pi^{-d / 2 } \frac{\gammaof{\frac{\alpha + d}{2}}}{\gammaof{1 - \frac{\alpha}{2}}}.
\end{align}
Lévy processes are characterized by their infinitesimal generator. A complete understanding of this notion is not absolutely necessary to understand our paper; however, the reader may consult \citet{schilling_introduction_2016,bottcher_levy_2013} for additional details on this topic.
For a Lévy process with triplet $(0, \Sigma, \nu)$, the infinitesimal generator is %
given by 
\begin{align}
    \label{eq:levy-generator}
    A u (x) := \frac1{2} \nabla \cdot \left(\Sigma \nabla u(x)\right) + \intrdzero \left( u(x+z) - u(x) - \nabla u(x) \cdot z \chi(\normof{z}) \right) \der \nu(z)
\end{align}
for any $u \in \mathcal{C}^2_b(\Rd)$.
In the case of the $\alpha$-stable L\'{e}vy process, it is well-known that this definition amounts to the fractional Laplacian, for which we give a definition below. We refer to \citet{daoud_fractional_2021,nezza_hitchhikers_2011} for the other equivalent definitions of the fractional Laplacian.

\begin{definition}[Fractional Laplacian]
    \label{def:fractional-laplacian}
For $\alpha\in (0,2)$, the fractional Laplacian of $u\in \cbtwo$ is given by 
       $- \fraclap u(x) := \calphad\lim_{\varepsilon \to 0} \int_{\Rd \backslash B_\varepsilon (0)} \frac{u(x+z) - u(x)}{\normof{z}^{d + \alpha}} \der z$.
    We can show that this definition is equivalent to \Cref{eq:levy-generator} in the case where $\Sigma=0$ and $\nu$ is given by \Cref{eq:calphad-levy-measure}, where it corresponds to the infinitesimal generator of $(\levy)_{t\geq 0}$.
\end{definition}

\begin{remark}\rm 
    The fractional Laplacian is sometimes defined without the constant $\calphad$ \citep{tristani_fractional_2013}. This corresponds to different normalizations and our choice is motivated by the fact that we want this operator to be the infinitesimal generator of $(\levy)_{t\geq 0}$.
\end{remark}

\subsection{Fractional Poincaré inequalities}
\label{sec:fractional-pi-background}

The analysis of differential privacy in the presence of Gaussian noise makes crucial use of the \textit{logarithmic Sobolev inequalities} (LSIs)
\citep{gross_logarithmic_1975-1,bakry_analysis_2014}, which are instrumental in making the proposed bounds uniform in time \citep{chourasia_differential_2021,chien_langevin_2025-1,ryffel_differential_2022,ye_differentially_2022}. In such cases, assuming that an LSI is satisfied is motivated by the fact that %
it is satisfied by the posterior distributions of the learning algorithms under reasonable assumptions. Unfortunately, as was noted by \citet{dupuis_generalization_2024} in their study of the generalization error of heavy-tailed SDEs, such LSIs may not always be applicable in the presence of $\alpha$-stable noise. 

In our paper, we show that, in the context of heavy-tailed dynamics, LSIs can be replaced by \textit{fractional Poincaré inequalities}.
Such inequalities have been widely studied \citep{wang_functional_2015,chafai_entropies_2004,mouhot_fractional_2009} and have been used in the context of machine learning for generalization bounds \citep{dupuis_generalization_2024} and sampling \citep{he_separation_2024-1}.
In particular, we have the following result \citep{gentil_logarithmic_2008,wu_new_2000}. 

\begin{theorem}[Fractional Poincaré inequalities]
    \label{thm:ht-poincare}
    Let $\mu$ be an infinitely divisible distribution on $\Rd$ with Lévy triplet $(0,\Sigma,\nu)$ in the sense of \Cref{eq:levy-khintchine}. Then, for any differentiable and $\mu$-square-integrable function $f : \Rd \to \R$, %
    \begin{align*}
        \intrd f^2 \der\mu - \left( \intrd f \der\mu \right)^2 \leq \int_\Rd \normof{\nabla f}^2_\Sigma \der \mu +  \intrdzero \intrd (f(x) - f(x+z))^2 \der \mu(x) \der\nu(z),
    \end{align*}
    where $\normof{u}_{\Sigma}^2 := u^{\top}\Sigma u$ for any $u \in \Rd$. In particular, if $\mu$ is the law of $L_1^\alpha$, then
\begin{align*}
        \intrd f^2 \der\mu - \left( \intrd f \der\mu \right)^2 \leq \calphad \intrdzero \intrd(f(x) - f(x+z))^2 \der \mu(x) \frac{\der z}{\normof{z}^{d + \alpha}}.
\end{align*}
\end{theorem}

In particular, when $\nu=0$, the above result recovers the classical Poincaré inequalities for Gaussian distributions.
By analogy with the generalizations of classical Poincaré inequalities from the Gaussian measures to more general ones, we propose the following definition.

\begin{definition}[$\alpha$-stable Poincaré inequalities]
    \label{def:ht-poincare}
    Let $\mu$ be a probability measure and $\alpha \in (0, 2)$, we say that $\mu$ satisfies an $\alpha$-stable Poincaré inequality with constants $(a,b)$, if for any  differentiable and $\mu$-square-integrable function $f : \Rd \to \R$, %
    \begin{align*}
        \intrd f^2 \der\mu - \left( \intrd f \der\mu \right)^2 
        \leq a C_{\alpha,d}\intrdzero\intrd\frac{(f(x) - f(x+z))^2}{\normof{z}^{d + \alpha}} \der \mu(x) \der z 
       + b \intrd \normof{\nabla f}^2 \der \mu.
    \end{align*}
\end{definition}

We use the unusual denomination $\alpha$-stable Poincaré inequality instead of fractional Poincaré inequality because we allow the inequality to contain a Gaussian component, corresponding to the term $\intrd \normof{\nabla f}^2 \der \mu$. This allows our analysis to apply to the case where the noise of the algorithm is a combination of Gaussian and $\alpha$-stable noises. We prove in \Cref{sec:stability-properties} several properties related to the stability of such inequalities under certain transformations, thus showing that these inequalities can be satisfied by a wide variety of probability distributions. For more general conditions regarding the validity of these inequalities, we refer to \citet{wang_functional_2015,mouhot_fractional_2009}.

In order to simplify the notation, we will use the following notation for Dirichlet forms.

\begin{definition}
\label{def:drichlet-forms}
For $\alpha \in (0,2)$, %
define
    \begin{align*}
        \ecal_{\alpha, \mu} (f,f) := \frac1{2} \calphad \intrd\intrd \frac{(f(x) - f(y))^2}{\normof{x - y}^{d + \alpha}} \der \mu(x) \der y,\quad f \in \cbtwo;
    \end{align*}
    and for $\alpha = 2$,
    \begin{align*}
        \ecal_{2,\mu} (f, f) := \intrd \normof{\nabla f(x)}^2 \der \mu (x),\quad  f \in \cbtwo.
    \end{align*}
\end{definition}
These quantities correspond, respectively, to the Dirichlet forms associated with $(\levy)_{t\geq 0}$ and $(\sqrt{2} B_t)_{t\geq 0}$. While a complete knowledge of Dirichlet forms is not required for this paper, we invite the reader to consult \citet{bakry_analysis_2014,bottcher_levy_2013} for more details.

\section{Differential Privacy of Lévy-Driven SDEs}
\label{sec:dp-of-heavy-tailed-sdes}

In this section, we present our main results, which include RDP guarantees for Lévy-driven SDEs. We first present our setup and key technical lemmas in \Cref{sec:renyi-flows,sec:setup-assumptions} respectively, before presenting two cases. First, we derive RDP guarantees under multifractal noise (\ie, a combination of Gaussian and $\alpha$-stable noise) in \Cref{sec:multifractal-case}. We will then explain how our analysis can be extended to pure-jump $\alpha$-stable noise in \Cref{sec:pure-jump-case}. 

\subsection{Setup and assumptions}
\label{sec:setup-assumptions}

Given a dataset $S \in \zcal^n$ and a loss function $\ell : \Rd \times \zcal \to \R_+$, we define the empirical risk as $\er(w) := n^{-1} \sum_{i=1}^n \ell(w,z_i)$. We also denote by $\mu_z$ the data distribution on $(\zcal, \fcal)$.

Let us consider two neighboring datasets $S\simeq S'\in \zcal^n$. Given a fixed common initial distribution, we consider the two following stochastic differential equations (SDEs):
\begin{equation}  \label{eq:mirror_sdes}\begin{cases}
       \der W_t = -\nabla \er (W_t) \der t + \sigma_\alpha \der \levy + \sigma_2 \sqrt{2} \der B_t,\\
       \der W_t' = -\nabla \erprime (W_t') \der t + \sigma_\alpha \der\levy + \sigma_2 \sqrt{2} \der B_t, 
\end{cases}\end{equation}
where $(B_t)_{t\geq 0}$ is a $d$-dimensional standard Brownian motion, $(\levy)_{t\geq 0}$ is a $d$-dimensional rotationally invariant $\alpha$-stable L\'{e}vy process with $\alpha \in (1,2)$, and $(\sigma_2,\sigma_\alpha)$ are positive constants.
Inspired by \cite{chourasia_differential_2021}, we define the gradient sensitivity as:
\begin{align}  
    \label{eq:gradient-sensitivity}
    \sensitivity_g := \esssup_{(z,z')\sim \mu_z \otimes \mu_z} \sup_{w\in \Rd} \normof{\nabla \ell(w,z') - \nabla \ell(w,z)}.
\end{align}
In some of our results, we make the following finite sensitivity assumption.
\begin{assumption}[Finite sensitivity]
    \label{ass:finite-sensitivity}
    The gradient sensitivity is finite, \ie, $ \sensitivity_g < +\infty$.
\end{assumption} %

\begin{remark}\rm
    In some of our main results, we use \Cref{ass:finite-sensitivity} to uniformly control the difference between the two drifts in \Cref%
    {eq:mirror_sdes}. This assumption is typically satisfied when $\ell(w,z)$ is a regularized loss of the form $\ell(w,z) = \ell_0(w,z) + \lambda \normof{w}^2$, where $\ell_0(\cdot, z)$ is Lipschitz-continuous, for all $z \in \zcal$. 
    Alternatively, one can impose this condition by clipping all the individual gradients $\nabla \ell(\cdot, z)$ outside a compact set.
    Similar conditions appear in the literature related to Gaussian noise \citep{chourasia_differential_2021,chien_langevin_2025-1,ryffel_differential_2022,ye_differentially_2022}, where it is often associated with projections on a compact set. 
    For Brownian motion--driven SDEs, reflection on convex sets is well-defined via the Skorokhod problem~\citep{lions1984stochastic}, which constrains the dynamics to a convex domain by means of a boundary regulator. For L\'evy processes, however, jumps can overshoot the boundary, so that the reflection is not canonical and typically induces non-local, state-dependent boundary conditions. Existing constructions of regulators are case specific and apply only in restricted settings ~\citep{menaldi1985reflected,costantini2005reflected,slominski2010jump,piera2008boundary}. For this reason, we consider projections only in the discrete-time case.
\end{remark}

\begin{remark}\rm
    In our proofs, the quantities $\normof{\nabla \ell(w,z') - \nabla \ell(w,z)}$ are integrated with respect to certain probability distributions (see \Cref{sec:omitted-proofs}). This suggests that \Cref{ass:finite-sensitivity} could be replaced by weaker but more intricate conditions. We leave this discussion for future work and focus our main results on the finite sensitivity case.
\end{remark}

We denote by $p_t$ and $p_t'$ the probability density functions of $W_t$ and $W_t'$. Throughout the paper, we assume that both SDEs are initialized from the same distribution. It is known that, under mild regularity assumptions \citep{duan_introduction_2015,umarov_beyond_2018}, $p_t$ and $p_t'$ are solutions (at least in a weak sense) of the following fractional Fokker-Planck equations,
\begin{equation}
    \label{eq:mirror_fpes}\begin{cases}
       \partial_t p_t = - \sigma_\alpha^\alpha \fraclap p_t + \sigma_2^2 \Delta p_t +  \nabla \cdot \left(p_t \nabla \er \right),\\
        \partial_t p_t' = - \sigma_\alpha^\alpha \fraclap p_t' + \sigma_2^2 \Delta p_t' + \nabla \cdot \left(p_t' \nabla\erprime \right),
\end{cases}
\end{equation}
where $\fraclap$ is the fractional Laplacian operator %
given in \Cref{def:fractional-laplacian}.

Our analysis of the differential privacy of heavy-tailed SDEs relies on estimates of the so-called \emph{Rényi flow}, \ie, the time-derivative of the Rényi divergence between $p_t$ and $p_t'$. In order to make these derivations perfectly rigorous, we need to impose enough regularity on the probability density functions of $W_t$ and $W_t'$, which is made precise as follows.

\begin{assumption}[Regularity conditions]
    \label{ass:regularity-conditions}
    Let $v_t := p_t / p_t'$. We assume that $p_t$, $p_t'$ and $v_t$ are positive, differentiable in $t$, and belong to $\cbtwo$. Moreover, we make the following \emph{non-explosion} assumption, \ie, the maps $t \mapsto \normof{v_t}_{\infty}$ are locally bounded on $\R_+$, and the functions $|\partial_t p_t|$ and $|\partial_t p_t'|$ are locally uniformly integrable on $\R_+$.
\end{assumption}

We say that a family of functions $\{f(t, \cdot)\}_{t\geq 0}$ is locally uniformly integrable if for all $t_0 \geq 0$, there exists $\varepsilon > 0$ such that $\sup_{t_0 - \varepsilon < t < t_0 + \varepsilon} |f(t, \cdot)| \in L^1(\der x)$.

We make \Cref{ass:regularity-conditions} for two reasons. First, it ensures that we can rigorously compute the entropy flow by differentiating $ \renyi{p_t}{p_t'}$ under the integrals. %
Secondly, the uniform integrability near $t=0$ ensures that the map $t\mapsto \renyi{p_t}{p_t'}$ is continuous at $t=0$, which simplifies the derivations.
Intuitively, the non-explosion condition ensures that the time-derivatives of the density functions of $W_t$ and $W_t'$ do not explode too fast, so that the derivative of their relative Rényi divergence can be controlled.
Intuitively, we observe that the above condition suggests that the drifts $\nabla \er$ and $\nabla \erprime$ do not grow too fast at infinity.

These conditions are reasonable in the context of SDEs driven by rotationally invariant $\alpha$-stable Lévy processes. Indeed, some studies show that, under mild conditions, the distributions of $(\levy)_{t\geq 0}$ and those generated by such SDEs have polynomial tails, typically in $\normof{x}^{-d-\alpha}$ \citep{samorodnitsky_tails_2003,blumenthal_theorems_1960,chen_heat_2017}.
This can be readily verified in the case of the $\alpha$-stable Ornstein-Uhlenbeck process.
Therefore, it is reasonable to expect that $v_t \in \cbtwo$.
We believe that this (reasonable) condition could be relaxed and we use this assumptions mainly to simplify some technical arguments, as checking such tail estimates is beyond the scope of this paper.

Finally, it should be noted that \Cref{ass:regularity-conditions} could be replaced by other conditions. Indeed, \citet{dupuis_generalization_2024} computed (different but) similar entropy flow along heavy-tailed SDEs, and their main theorems make use of intricate regularity and domination conditions of the functions $v_t(x)$. In our work, we simplify this approach by noting that \Cref{ass:regularity-conditions} is enough for the Rényi flow computations to be rigorous.

\subsection{R\'{e}nyi flows computations}
\label{sec:renyi-flows}

Our work is based on a new upper bound of the Rényi flow, \ie, the time derivative of the Rényi divergence along the fractional Fokker-Planck equations.  With a slight abuse of notation, in what follows, we denote by $p'_t$ the probability distribution whose density function is also denoted by $p'_t$.

\begin{restatable}[Rényi flow for heavy-tailed SDEs]{theorem}{thmRenyiFlow}
    \label{thm:renyi-flow-with-Dirichlet-forms}
   Assume that \Cref{ass:regularity-conditions} holds.
    Then, for any $\beta\ge2$ and $t>0$,
    \begin{align*}
        \timeder \renyi{p_t}{p_t'} \leq -\frac{2\sigma_\alpha^\alpha }{\beta - 1} \frac{ \ecal_{\alpha, p_t'}\left(v_t^{\beta / 2}, v_t^{\beta/2}\right)}{\ebeta \left(p_t, p_t'\right)} - \frac{4\sigma^2_2}{\beta} \frac{\ecal_{2, p_t'}\left(v_t^{\beta / 2}, v_t^{\beta/2}\right)}{\ebeta \left(p_t, p_t'\right)} + \mathrm{R}_{\mathrm{potential}}, 
    \end{align*}
    where the following term accounts for the contribution of the potentials
    \begin{align*}
\mathrm{R}_{\mathrm{potential}} :=  \beta \intrd v_t^{\beta - 1} \frac{  \langle \nabla v_t, \nabla \widehat{\mathcal{R}}_{S'} - \nabla \er  \rangle }{\ebeta \left(p_t, p_t'\right)} p_t' \der x.
    \end{align*}
\end{restatable}

\begin{proof}
    See \Cref{sec:omitted-proofs-renyi-flows}.
\end{proof}

\Cref{thm:renyi-flow-with-Dirichlet-forms} shows that the Rényi flow is controlled by three terms: \textrm{(i)} the Dirichlet form $\ecal_{\alpha,p_t'}$ associated with the pure-jump part of the noise, \textrm{(ii)} the Dirichlet form $\ecal_{2,p_t'}$ corresponding to the diffusive part of the noise, and \textrm{(iii)} the quantity $ \mathrm{R}_{\mathrm{potential}}$, which is the contribution from the drift difference of both SDEs.

In the absence of pure-jump noise (\ie, $\sigma_\alpha = 0$), \Cref{thm:renyi-flow-with-Dirichlet-forms} extends existing computations obtained in the case of Gaussian noise \citep{chourasia_differential_2021}. 
In the case $\beta=2$, related computations can be found in \cite{dupuis_generalization_2024,he_separation_2024-1}, but without being related to differential privacy (in \cite{he_separation_2024-1}, the simpler case of fractional heat flows is considered).
However, the extension to arbitrary $\beta \geq 2$ is a major obstacle, as the Dirichlet form $\ecal_{\alpha,p_t'}$ naturally appears in the derivations when $\beta = 2$. We address this issue in the proof of \Cref{thm:renyi-flow-with-Dirichlet-forms} and directly compare the Rényi flow to $\ecal_{\alpha,p_t'}$, for $\beta \geq 2$.

The following lemma allows us to apply heavy-tailed Poincaré inequalities to the entropy flow computed in \Cref{lemma:renyi-flow-two dynamics}.
It is a key component of our analysis.
\begin{restatable}{lemma}{lemmaBregmanLowerBound}
    \label{lemma:bregman-integral-lower-bound}
     Assume that \Cref{ass:regularity-conditions} holds.
    Assume that, for all $t>0$, $p_t'$ satisfies an $\alpha$-stable Poincaré inequality with constants $\left(\gamma \sigma_\alpha^\alpha, \gamma \sigma_2^2\right)$ for some $\gamma>0$. Then, for $\beta\ge2$ and $t>0$, as long as $\ebeta (p_t, p_t') < \infty$,
    \begin{align*}
\frac{2\sigma_\alpha^\alpha}{\beta - 1} \ecal_{\alpha, p_t'}\left(v_t^{\beta / 2}, v_t^{\beta/2}\right) + \frac{2 \sigma_2^2}{\beta} \ecal_{2, p_t'}\left(v_t^{\beta / 2}, v_t^{\beta/2}\right) \geq  \frac{1}{\gamma \beta} \ebeta (p_t, p_t') \left( 1 - e^{-\renyi[\beta]{p_t}{p_t'}} \right).
    \end{align*}
    Moreover, when $\sigma_2 = 0$, the constant $1 / (\gamma \beta)$ above can be replaced by $1 / (\gamma (\beta - 1))$.
\end{restatable}

\begin{proof}
    Deferred to \Cref{sec:omitted-proofs-renyi-flows}.
\end{proof}
\vspace{-6mm}

\subsection{Differential privacy of multifractal SDEs}
\label{sec:multifractal-case}

Based on the technical lemmas of the previous subsection, we can now derive our main results, which consists in RDP guarantees for \Cref{eq:mirror_fpes}.
We first present the case where the noise has a non-trivial Gaussian component, \ie, when $\sigma_2 > 0$, which we refer to as the \emph{multifractal} setting. 
This case has the advantage of being technically simpler, due to the regularizing effect of the Gaussian noise.

The following theorem is a differential privacy bound for multifractal SDEs. 

\begin{theorem}
    \label{thm:multifractal-case-finite-sensitivity}
    Let $\beta \geq 2$.
    Suppose \Cref{ass:finite-sensitivity,ass:regularity-conditions} hold and $p_t'$ satisfies an $\alpha$-stable Poincaré inequality with constants $(\gamma \sigma_\alpha^\alpha, \gamma \sigma_2^2)$ for some $\gamma>0$ and for all $t > 0$. 
    Then,
\begin{align}
    \label{eq:dp-multifractal-concentrated}
    \renyi{p_t}{p_t'} \leq \frac{ \beta \sensitivity_g^2}{2 \sigma_2^2 n^2}t=:K_nt, \quad t > 0.
\end{align} 
If moreover $K_n < a:=\frac{1}{\gamma \beta}$, then,
\begin{align}
    \label{eq:dp-multifractal-with-poincaré}
     \renyi{p_t}{p_t'} \leq -\log \left( 1 - \frac{\gamma \sensitivity_g^2 \beta^2}{2 \sigma_2^2 n^2} \right),\quad t>0.
\end{align}
\end{theorem}

\begin{proof}
    By \Cref{thm:renyi-flow-with-Dirichlet-forms}, for $t> 0$,
     \begin{align*}
        \timeder \renyi{p_t}{p_t'} \leq -\frac{2\sigma_\alpha^\alpha }{\beta - 1} \frac{\ecal_{\alpha, p_t'}\left(v_t^{\beta / 2}, v_t^{\beta/2}\right)}{\ebeta (p_t, p_t')} - \frac{4\sigma^2_2}{\beta} \frac{\ecal_{2, p_t'}\left(v_t^{\beta / 2}, v_t^{\beta/2}\right)}{\ebeta (p_t, p_t')} + \mathrm{R}_{\mathrm{potential}}. 
    \end{align*}
    By the Cauchy-Schwarz and Young inequalities, we have that for all $\lambda > 0$,
    \begin{align*}
        \mathrm{R}_{\mathrm{potential}} &=  \beta \intrd v_t^{\beta - 1} \frac{ \left\langle \nabla v_t, \nabla \widehat{\mathcal{R}}_{S'} - \nabla \er \right\rangle }{\ebeta (p_t, p_t')} p_t' \der x \\
        &\leq \frac{\beta}{\ebeta (p_t, p_t')} \left( \frac{\lambda}{2} \intrd v_t^{\beta - 2} \normof{\nabla v_t}^2 p_t' \der x + \frac{1}{2 \lambda} \intrd \normof{\nabla \er - \nabla \erprime}^2 v_t^\beta p_t' \der x  \right) \\
        &\leq \frac{2 \lambda}{\beta} \frac{\ecal_{2, p_t'}\left(v_t^{\beta / 2}, v_t^{\beta/2}\right)}{\ebeta (p_t, p_t')} + \frac{ \beta \sensitivity_g^2}{2 \lambda n^2}.
    \end{align*}
    In the equation above, we choose $\lambda := \sigma_2^2$, which gives
     \begin{align*}
        \timeder \renyi{p_t}{p_t'} \leq -\frac{2\sigma_\alpha^\alpha }{\beta - 1} \frac{\ecal_{\alpha, p_t'}\left(v_t^{\beta / 2}, v_t^{\beta/2}\right)}{\ebeta (p_t, p_t')} - \frac{2\sigma^2_2}{\beta} \frac{\ecal_{2, p_t'}\left(v_t^{\beta / 2}, v_t^{\beta/2}\right)}{\ebeta (p_t, p_t')} + \frac{ \beta \sensitivity_g^2}{2 \sigma_2^2 n^2}. 
    \end{align*}
    Now we use \Cref{lemma:bregman-integral-lower-bound} and obtain that
    \begin{align}
        \label{eq:diffential-inequality-multifractal}
        \timeder \renyi{p_t}{p_t'} \leq  -\frac{1}{\gamma \beta} \left( 1 - e^{-\renyi[\beta]{p_t}{p_t'}} \right) + \frac{ \beta \sensitivity_g^2}{2 \sigma_2^2 n^2}.
    \end{align}
    We solve this type of differential inequality in \Cref{lemma:differential-exponential-inequality} in Appendix~\ref{sec:technical lemmas}.
    In order to apply this lemma, we observe that \Cref{ass:regularity-conditions} also implies the continuity of $t \mapsto \renyi{p_t}{p_t'}$ on $[0,\infty)$, by the dominated convergence theorem.
    The claim immediately then follows by noting that both SDEs are initialized with the same probability distributions.
\end{proof}

The differential inequality \eqref{eq:diffential-inequality-multifractal} is a consequence of the fact that our analysis relies on the fractional Poincaré inequality, as the logarithmic Sobolev inequality is not available in our case (see \Cref{sec:fractional-pi-background}). 
A similar differential inequality appears in the derivations of \citet{cao_exponential_2019} in their study of exponential decay of R\'{e}nyi entropy along (Gaussian) Fokker-Planck equation. In the particular case where $\sensitivity_g = 0$, if we set $p_t'$ to be the invariant distribution of \Cref{eq:mirror_fpes}, then \Cref{eq:diffential-inequality-multifractal} ensures the exponential decay of R\'{e}nyi entropies along fractional Fokker-Planck equations, which is of independent interest.

We distinguish two regimes: \textrm{(i)} %
it always holds that $\renyi{p_t}{p_t'}$ grows at most linearly with time, corresponding to a regime where the Poincaré constant of the SDEs is too large, and \textrm{(ii)} if additionally $K_n<a$, then $\renyi{p_t}{p_t'}$ is bounded by a constant for all $t>0$. Note that, as long as $\gamma < \infty$, we have $K_n<a$ as long as the sample size $n$ is sufficiently large.

\Cref{thm:multifractal-case-finite-sensitivity} has an interesting dependence on $\beta$ (the order of the Rényi divergence). By \Cref{eq:dp-multifractal-concentrated}, we always have a RDP guarantee in $\rho \beta$, with $\rho = \landau{t / n^2}$. This dependence in $\beta$ corresponds to the zero-concentrated DP guarantee introduced by \citet{bun_concentrated_2016} but comes at a cost of a potential linear dependence on time. On the other hand, we observe that when $K_n < a$, we can remove the time-dependence at the cost of a slightly worst dependence on $\beta$, yielding a bound in $\landau{\beta^2 / n^2}$. It can be seen from our proofs (in \Cref{lemma:bregman-integral-lower-bound}) that the presence of $\beta^2$ is inherently due to the non-local nature of the Dirichlet form $\ecal_{\alpha,p_t'}$, and cannot be avoided in our approach. We see this behavior as a \emph{semi-concentrated} DP guarantee.

In the case $K_n<a$, we obtain $(\beta,\gamma)$-RDP with $\gamma = \landau{\beta^2 / n^2}$ as $n \to \infty$. By \Cref{lemma:rdp-to-zero-delta-DP}, this implies $(0, \delta)$-DP with $\delta = \landau{n^{-1}}$, which is the same rate in $n$ that was obtained by \citet{simsekli_differential_2024} for (S)GD under heavy-tailed noise. 
Moreover, our results provide a differential privacy guarantee that is independent of the dimension.
This independence of the dimension is intrinsically related to the presence of Gaussian noise. In the next section, we will analyze in more detail the case of pure-jump noise and show that we still obtain a weaker dependence on the dimension compared to existing works.

\subsection{Differential privacy of pure-jump SDEs}
\label{sec:pure-jump-case}

We now present the case of pure $\alpha$-stable noise. That is, $\sigma_2 = 0$, and we consider the following SDEs, where we denote $\sigma := \sigma_\alpha$ for simplicity,
\begin{equation} \label{eq:mirror_sdes_pure_jump}\begin{cases}   
       \der W_t = -\nabla \er (W_t) \der t + \sigma \der\levy, \\
       \der W_t'  = -\nabla \erprime (W_t') \der t + \sigma \der\levy. 
\end{cases}\end{equation}

Before presenting the main result of this subsection, we first motivate the approach. 
In the proof of \Cref{thm:multifractal-case-finite-sensitivity}, the Young inequality is used, leading to a term proportional to $\ecal_{2,p_t'} \left(v_t^{\beta / 2}, v_t^{\beta / 2}\right)$, which is canceled out by the diffusive part of the Rényi flow. Unfortunately, this procedure cannot be applied when $\sigma_2 = 0$. 

The intuition to address this issue comes from the celebrated Bourgain-Brezis-Mironescu's theorem (also known as BBM formula) \citep{bourgainBrezisMironescu2001}, stating that for all $u$ in the Sobolev space $H^1 \left(\Rd\right) := \left\{ u \in L^2\left(\Rd\right): \nabla u \in L^2\left(\Rd\right) \right\}$, 
\begin{align}
    \label{eq:bbm-formula}
    \lim_{\alpha \to 2^-} \left( 1 - \frac{\alpha}{2} \right) \iint_{\Rd \times \Rd} \frac{|u(x) - u(y)|^2}{\normof{x - y}^{d + \alpha}} \der x \der y = \frac{\pi^{d / 2}}{2 \gammaof{1 + \frac{d}{2}}} \intrd \normof{\nabla u (x)}^2 \der x.
\end{align}
In our case, the Dirichlet form $\ecal_{\alpha, p_t'}(u,u)$ corresponds to a weighted version of the left-hand side of \Cref{eq:bbm-formula}. In fact, we are even able to prove a weighted version of the BBM formula; namely, for any $u \in \cbtwo$ and any Borel probability measure $\mu \in \mathcal{P}(\Rd)$,  
\begin{align}
    \label{eq:weighted-BBM}
    \lim_{\alpha \to 2^-} \ecal_{\alpha, \mu}(u,u) = \ecal_{2, \mu}(u,u).
\end{align}
The proof of this result is deferred to \Cref{lemma:weighted-BBM} in \Cref{sec:omitted-proofs-pure-jump-jv-asssumptions}. Note that a similar observation can be found in \citep{he_separation_2024-1}. \Cref{eq:weighted-BBM} shows that it is reasonable to approximate $\ecal_{\alpha, p_t'}$ by $\ecal_{2, p_t'}$ and, hence, makes the proof technique of \Cref{thm:multifractal-case-finite-sensitivity} applicable to our case. 
In order to leverage this intuition more quantitatively and non-asymptotically, we require the following additional regularity assumption.

\begin{assumption}
    \label{ass:c2-bounded-assumption}
    We assume that, with $v_t := p_t / p_t'$, %
    for some $T>0$ and all $\beta\ge 2$, 
    \begin{align*}  
        \sup_{S \simeq S'} \sup_{t\leq T}\left(\left\|\nabla v_t^{\beta / 2}\right\|_\infty+\left\|\nabla^2 v_t^{\beta / 2}\right\|_\infty\right)  < \infty  \quad \text{and } \quad  \inf_{S \simeq S'  } \inf_{t \leq T } \left\Vert{\nabla v_t^{\beta / 2}} \right\Vert_{L^2(p_t')} >  0.
    \end{align*}
\end{assumption}

Equipped with this assumption, we can prove the following result, which is a RDP guarantee for SDEs driven by pure-jump $\alpha$-stable noise.

\begin{restatable}{theorem}{thmPureJumpCase}
\label{thm:dp_guarantees_sdes_pure_jump}
Assume that \Cref{ass:finite-sensitivity,ass:regularity-conditions,ass:c2-bounded-assumption} hold, and that, for all $t\in (0,T]$, $p_t'$ satisfies an $\alpha$-stable Poincaré inequality with constants $(\gamma \sigma^\alpha, 0)$ for some $T,\gamma>0$.
Then, for $\beta\ge 2$, there exists $R>0$ (which can depend on $d$, $\beta$, $\alpha$, $T$) such that with the following constants 
\begin{align*}
    a := \frac{1}{2 \gamma (\beta - 1)} , \quad \kalphad := \frac{4 (2 - \alpha) d \gammaof{\frac{d}{2}} \gammaof{1 - \frac{\alpha}{2}}}{ \alpha 2^{\alpha} R^{2 - \alpha}\gammaof{\frac{d + \alpha}{2}}} , \quad K_n := \frac{\kalphad (\beta-1) \sensitivity_g^2}{\sigma^\alpha n^2},
\end{align*}
we have
\begin{align*}
     \renyi{p_t}{p_t'} \leq  \frac{\kalphad (\beta-1) \sensitivity_g^2 t}{\sigma^\alpha n^2}\quad  \text{for all $t\in [0,T]$}.
\end{align*}
If moreover $K_n < a$, then
\begin{align*}
     \renyi{p_t}{p_t'} \leq - \log \left( 1 - \frac{2 \gamma (\beta - 1)^2 \kalphad \sensitivity_g^2}{ \sigma^{\alpha} n^2} \right)  \quad\text{for all $t\in [0,T]$}.
\end{align*}
\end{restatable}

\begin{proof}
    Deferred to \Cref{sec:omitted-proofs-pure-jump-jv-asssumptions}.
\end{proof}

The main technical ingredient in the proof of \Cref{thm:dp_guarantees_sdes_pure_jump} is to obtain a more quantitative version of the weighted BBM formula obtained above.

In \Cref{thm:multifractal-case-finite-sensitivity}, we observe the same two regimes as in the multifractal case, discussed in \Cref{sec:multifractal-case}. Therefore, the same remarks reagarding these two regimes apply. However, this result closely exploits the structure of the $\alpha$-stable noise, which translates into an interesting dependence on $\alpha$ and $d$, which we will analysis later.

We provide in \Cref{lemma:constants-asymptotics} in Appendix~\ref{sec:technical lemmas} an asymptotic analysis of the constant $\kalphad$, showing that $\kalphad = \landau{1}$ as $\alpha \to 2^-$, as long as $R$ does not depend on $\alpha$ %
or does not explode when $\alpha \to 2^-$. %
In particular, this shows that the effect of the unknown constant $R$ vanishes as $\alpha \to 2^-$, as expected from the weighted BBM formula. 
On the other hand, if we assume that $R$ grows slowly with $d$, in the high dimension limit ($d \to \infty$), \Cref{lemma:constants-asymptotics} provides
\begin{align*}
    \kalphad = \landau[d \to \infty]{\frac{d^{1-\alpha / 2}}{n^2 \sigma^\alpha}}.
\end{align*}
As $\alpha < 2$ (and in practice $\alpha$ is close to $2$, see \citep{barsbey_heavy_2021}) the dimension always has an exponent smaller than $1$, which is therefore easily compensated for large sample sizes by the term $n^{-2}$. 
Moreover, we observe that the dimension dependence vanishes as $\alpha \to 2^-$, hence recovering the known behavior in the presence of Gaussian noise.

Moreover, if $R$ grows slowly with $d$, \Cref{lemma:rdp-to-zero-delta-DP}, together with \Cref{thm:dp_guarantees_sdes_pure_jump} implies a $(0,\delta)$-DP guarantee with
\begin{align}
    \label{eq:asymptotics-dp-delta-version}
    \delta = \landau[d \to \infty]{\frac{d^{\frac{2 - \alpha}{4}}}{n \sigma^{\alpha / 2}} } = \landau[d \to \infty]{\frac{\sqrt{d}}{n} \left( \sigma \sqrt{d} \right)^{-\alpha / 2} }.
\end{align}
Again, we recover the dependence on $n$ of \citet{simsekli_differential_2024}, obtained in their study of heavy-tailed (S)GD. However, our result has a better dependence on $d$, as these authors obtained an estimate in $\landau{d^{(1 + \alpha) / 2}}$.
However, it should be noted that the results of \citet{simsekli_differential_2024} do not require finite sensitivity, which we need in our study. 
This reveals a new tradeoff between the dimension dependence and the finite sensitivity assumption for heavy-tailed algorithms.
Finally, while \citet{simsekli_differential_2024} suggested that heavier tails systematically lead to better guarantees, our bound reveals a more complex structure, indicating that heavier tails might be beneficial (resp. harmful) when $\sigma \sqrt{d} < 1$ (resp. $\sigma \sqrt{d} > 1$). 
A similar phenomenon has been reported by \citet{dupuis_generalization_2024} for generalization bounds.
\begin{remark} \rm 
    The constant $R$ appearing in the above theorem is of a different nature from that appearing in the entropy flow computations of \citet{dupuis_generalization_2024}. 
    Indeed, these authors relied on an intricate assumption (Assumption 4.3 in their paper) that directly imposes a lower bound on a Dirichlet form-like quantity. 
    In our setting, we improve this analysis by directly obtaining a lower bound on the Dirichlet form under \Cref{ass:c2-bounded-assumption}.
\end{remark}

\section{Extension to the Discrete-Time Setting}
\label{sec:discrete-time}

In this section, we discuss the extension of our results to the discrete-time setting, \ie, (stochastic) gradient descent with heavy-tailed noise. 
We first present our discrete-time RDP guarantees in \Cref{sec:dp-heavy-tailed-gradient-descent}. Then, \Cref{sec:stability-properties} focuses on the analysis of the associated Poincaré inequalities, thus providing theoretical foundations for our main technical assumptions.

\subsection{Differential privacy of heavy-tailed (stochastic) gradient descent}
\label{sec:dp-heavy-tailed-gradient-descent}

In this subsection, we analyze the RDP guarantees of gradient descent under $\alpha$-stable noise. These results can be easily extended to the case of stochastic gradient descent (\ie, with mini-batch noise) and to the case of multifractal noise, in which case the constants obtained should be replaced by those appearing in \Cref{thm:multifractal-case-finite-sensitivity}.

\paragraph{Setup.} Let $\mathcal{C} \subset \Rd$ be a closed convex set, and $S\simeq S'$ be fixed neighboring datasets in $\zcal^n$ for some $n\ge1$, denoted $S = (z_1,\dots,z_n)$ and $S' = (z_1',\dots,z_n')$. Let $b \in \{1,\dots, n\}$ be a fixed batch size. Given a random mini-batch of indices $\Omega_k \subset \{1,\dots, n\}$ with $|\Omega_k| = b$, we denote the stochastic gradients by
\begin{align*}
    \widehat{g}_S(x, \Omega_k) := \frac1{b} \sum_{i\in\Omega_k} \nabla \ell (w, z_i), \quad \widehat{g}_{S'}(x, \Omega_k) := \frac1{b} \sum_{i\in\Omega_k} \nabla \ell (w, z_i').
\end{align*}
In our study, we consider that $\Omega_k$ is a uniformly sampled random subset of $\{1,\dots, n\}$ of size $b$, which is the same sampling scheme as in prior works \citep{simsekli_differential_2024}. Moreover, we assume that $(\Omega_k)_{k\in\N}$ are \iid and that they are independent of the other random variables (\eg, the Lévy processes and Brownian motions).

We consider the following recursions:
\begin{equation}
    \label{eq:discrete-decursion-pure-jump}\begin{cases}
          X_{k+1} = \Pi_{\mathcal{C}} \left(X_k - \eta \widehat{g}_S(X_k, \Omega_k) + \sigma \eta^{1/\alpha} \xi_k \right),\\
        X_{k+1}' = \Pi_{\mathcal{C}} \left(X_k' - \eta \widehat{g}_{S'}(X_k', \Omega_k) + \sigma \eta^{1/\alpha} \xi_k \right),
\end{cases}\end{equation}
where $\eta>0$, $\xi_k \sim \mathcal{S}\alpha\mathcal{S} := \mathrm{Law}(L_1^\alpha)$, $k\ge0$, are \iid $\alpha$-stable random vectors in $\Rd$ with $\alpha \in (1,2)$, and $\Pi_{\mathcal{C}}$ denotes the orthogonal projection onto $\mathcal{C}$. Let $t_k := k \eta$ for all $k\ge0$. %
In order to define the interpolation scheme, we define the following random mappings,
\begin{align*}
    F_1^{(k)}(x) := \frac1{2} \left( \widehat{g}_S(x, \Omega_k) + \widehat{g}_{S'}(x, \Omega_k) \right), \quad F_2^{(k)}(x) := \frac1{2} \left( \widehat{g}_S(x, \Omega_k) - \widehat{g}_{S'}(x, \Omega_k) \right).
\end{align*}
Inspired by \citet{chourasia_differential_2021}, we construct a time-continuous stochastic process $\Theta_t$ (resp. $\Theta_t'$) by the following interpolation mechanism:
\begin{empheq}[left = \empheqlbrace]{align*}
       &\Theta_t = \Theta_{t_k} - \eta F_1^{(k)}(\Theta_{t_k}) - (t - t_k) F_2^{(k)}(\Theta_{t_k}) + \sigma (t - t_k)^{1/\alpha} \xi_k, ~~ t_k \leq t < t_{k+1},\\
       &\Theta_t' = \Theta'_{t_k} - \eta F_1^{(k)}\left(\Theta'_{t_k}\right) + (t - t_k) F_2^{(k)}\left(\Theta'_{t_k}\right) + \sigma (t - t_k)^{1/\alpha} \xi_k, ~~ t_k \leq t < t_{k+1},\\
        &\Theta_{t_{k+1}}^- := \Pi_{\mathcal{C}} \left( \lim_{t \to \eta^-} \Theta_{t_k + t} \right), \quad \Theta_{t_{k+1}}^{'-} := \Pi_{\mathcal{C}} \left( \lim_{t \to \eta^-} \Theta'_{t_k + t} \right).
\end{empheq}
We see that for all $k \in \mathds{N}$, $\Theta_{t_k}^- \sim \mathrm{Law} (X_k)$, and similarly for $\Theta_{t_k}^{'-}$.

Let $p_t$ and $p'_t$ denote the probability density function of $\Theta_t$ and $\Theta'_t$, respectively. Under mild regularity conditions, we can show that they satisfy the following Fokker-Planck equations \citep{ryffel_differential_2022,chourasia_differential_2021} for $t_k < t < t_{k+1}$,
$$
\left\{
    \begin{aligned}
        \partial_t p_t (\theta) &= - \sigma^\alpha \fraclap p_t(\theta) + \nabla \cdot \left(p_t(\theta) V_k(t,\theta) \right) \\
        \partial_t p_t' (\theta) &= - \sigma^\alpha \fraclap p_t'(\theta) + \nabla \cdot \left(p_t'(\theta) V_k'(t,\theta) \right)
    \end{aligned}
\right. \text{ with }  \left\{
    \begin{aligned}
        &V_k(t, \theta) = \Eof{F_2^{(k)}(\Theta_{t_k}) \big| \Theta_t = \theta}, \\
        &V_k'(t, \theta) = -\Eof{F_2^{(k)}(\Theta'_{t_k}) \big| \Theta_t' = \theta}.
    \end{aligned}
\right. 
$$

Because of the presence of the projection $\Pi_{\mathcal{C}}$, we adapt \Cref{ass:finite-sensitivity} as follows, hence weakening this condition.
\begin{assumption}
    \label{ass:finite-sensitivity-on-convex}
    We assume that the sensitivity is finite on $\mathcal{C}$, \ie, 
    \begin{align*}
       \mathcal{S}_{g,\mathcal{C}}:= \esssup_{z,z'\sim\mu_z \otimes \mu_z} \sup_{w \in \mathcal{C}} \normof{\nabla \ell(w,z) - \nabla \ell(w,z')} < +\infty.
    \end{align*}
\end{assumption}

\begin{restatable}{theorem}{thmDiscretePureJump}
    \label{thm:discrete-pure-jump}
    Assume that \Cref{ass:regularity-conditions,ass:finite-sensitivity-on-convex,ass:c2-bounded-assumption} hold and that, for all $t>0$,  $p_t'$ satisfies an $\alpha$-stable Poincaré inequality with constant $(\gamma \sigma^\alpha, 0)$ for some $\gamma>0$. %
    Finally assume that $t\mapsto \renyi{p_t}{p_t'}$ is right-continuous.
    Then, there exists a constant $R$ (which can depend on $d$, $\beta$, $\alpha$ and $\eta$) such that, with the following constants
    \begin{align*}
        a := \frac1{2\gamma (\beta - 1)} , \quad \kalphad := \frac{4 (2 - \alpha) d \gammaof{\frac{d}{2}} \gammaof{1 - \frac{\alpha}{2}}}{ \alpha 2^{\alpha} R^{2 - \alpha}\gammaof{\frac{d + \alpha}{2}}} , \quad K_n := \frac{\kalphad (\beta-1)    \mathcal{S}_{g,\mathcal{C}}^2}{\sigma^\alpha n^2},
    \end{align*}
    we have  
    \begin{align*}
        \renyi{\mathrm{Law}(X_k)}{\mathrm{Law}(X_k')} \leq  \frac{\kalphad (\beta-1)   \mathcal{S}_{g,\mathcal{C}}^2}{\sigma^\alpha n^2} k \eta, \quad k \in \mathds{N}.
    \end{align*}
    If moreover $K_n < a$, then
    \begin{align*}
        \renyi{\mathrm{Law}(X_k)}{\mathrm{Law}(X_k')} \leq -\log \left( 1 - \frac{2 \gamma (\beta - 1)^2 \kalphad \sensitivity_g^2}{\sigma^\alpha n^2} \right) ,\quad k \in \mathds{N}
    \end{align*}
\end{restatable}

\begin{proof}
    Deferred to \Cref{sec:proofs-discrete-pure-jump}.
\end{proof}
\vspace{-1mm}
This theorem shows that, under our assumptions, we obtain similar guarantees for heavy-tailed gradient descent as in the continuous case. 
Thus, as in the discussion of \Cref{thm:dp_guarantees_sdes_pure_jump}, we see that \Cref{thm:discrete-pure-jump} implies a $(0,\delta)$-DP guarantee that has the same dependence in $n$ as in previous works, but with a much better dependence on the dimension $d$.

This improvement comes at the cost of assuming that the probability distributions generated by the learning algorithm satisfy a fractional Poincaré inequality. In the next subsection, we investigate in more detail the behavior of such inequalities in the context of heavy-tailed gradient descent.

\subsection{Stability properties of the fractional Poincaré inequalities}
\label{sec:stability-properties}

In this section, we study the stability of fractional Poincaré inequalities through the noisy (S)GD iterations. 
This section is a sanity check. It ensures that the $\alpha$-stable Poincaré inequalities can be satisfied in practice.
In the case of pure Gaussian noise, similar stability properties have already been used in the context of differential privacy \citep{chourasia_differential_2021,chien_langevin_2025-1,ryffel_differential_2022}.
These existing results are based on two properties of the classical Poincaré inequalities: (1) stability by Gaussian convolution and (2) stability under pushforward by Lipschitz mapping. It is also well known that Poincaré inequalities are stable by bounded perturbations \citep{holley-stroock,bakry_analysis_2014}. In our paper, we show that fractional Poincaré inequalities also satisfy comparable properties, which might be of independent interest beyond this work. 

The stability of fractional Poincaré inequalities under convolution with an $\alpha$-stable distribution was already observed by \citet{he_separation_2024-1}. We extend this result in the following lemma to the $\alpha$-stable Poincaré inequalities defined in \Cref{def:ht-poincare}. The proof is similar, but we present it for the sake of completeness.

\begin{restatable}[Stability under convolution with stable distributions]{lemma}{lemmaConvolutionStability}
    \label{lemma:stability-convolution}
    Let $\mu$ and $\mu'$ be two probability measures on $\Rd$ with densities $\mu(x)$ and $\mu'(x)$, respectively.
    Assume that $\mu$ and $\mu'$ satisfy the $\alpha$-stable Poincaré inequality with constants $(\gamma_1,\gamma_2)$ and $\left(\gamma_1',\gamma_2'\right)$, respectively.
    Let us denote $ m:= \mu \ast \mu'$.  
    Then, $m$ satisfies an $\alpha$-stable Poincaré's inequality with constants $\left(\gamma_1 + \gamma_1',\gamma_2 + \gamma_2'\right)$.
\end{restatable}

\begin{proof}
    See \Cref{sec:proofs-stability-properties}.
\end{proof}

\begin{remark}\rm
    This result recovers as a particular case the stability by convolution properties of the classical Poincaré inequalities (\ie, when $\gamma_1 = \gamma_1'= 0$) \citep{bakry_analysis_2014}.
\end{remark}

Regarding the stability under Lipschitz mappings, the situation is more contrasted. This phenomenon is due to the non-local nature of the Dirichlet forms $\ecal_{\alpha, \mu}$ when $\alpha < 2$. Indeed, we find that  bi-Lipschitz continuity is needed in the case of fractional Poincaré inequalities.
This leads to the following stability property of $\alpha$-stable Poincaré inequalities under bi-Lipschitz diffeomorphisms. This result is new to the best of our knowledge.

\begin{restatable}[Stability under bi-Lipschitz diffeomorphisms]{lemma}{lemmaBiLipschitzStability}
    \label{lemma:bi-lipschitz-stability}
    Assume that $\mu$ has a density $\mu(x)$ with respect to the Lebesgue measure and satisfies an $\alpha$-stable Poincaré inequality with constants $(\gamma_1,\gamma_2)$. Let $T:\Rd \longrightarrow\Rd$ be a $\mathcal{C}^1$ bi-Lipschitz diffeomorphism, i.e., $T$ is continuously differentiable and there exist $L_1,L_2 > 0$ such that for all $x,y\in \R^d$
    \begin{align*}
        L_1 \normof{x - y} \leq \normof{T(x) - T(y)} \leq L_2 \normof{x - y}.
    \end{align*}
    Then the measure $T_{\#}\mu$ satisfies an $\alpha$-stable Poincaré inequality with constants $\left( \gamma_1  \frac{L_2^{\alpha + d}}{L_1^d} ,\gamma_2 L_2^2 \right)$.
\end{restatable}

\begin{proof}
    See \Cref{sec:proofs-stability-properties}.
\end{proof}
Finally, we can prove a perturbation lemma, similar to the standard Holley-Stroock lemma \citep{holley-stroock}. Even though this result is not immediately used in the proofs of our main results, we present it to give a complete picture of the stability properties of the fractional Poincaré inequalities. %

\begin{restatable}[Bounded perturbation]{lemma}{lemmaBoundedPerturbation}
    \label{lemma:stability-bounded-perturbation}
    Assume that $\mu$ has a density $\mu(x)$ with respect to the Lebesgue measure and satisfies an $\alpha$-stable Poincaré inequality with constants $(\gamma_1,\gamma_2)$. Let $\mu' \leq \mu$ be another Borel probability measure such that, almost surely, $e^{-b} \leq \der \mu' / \der \mu \leq e^b$
    for some $b\geq 0$. Then $\mu'$ satisfies an $\alpha$-stable Poincaré inequality with constants $\left(e^{2b} \gamma_1, e^{2b} \gamma_2\right)$.
\end{restatable}

\begin{proof}
    See \Cref{sec:proofs-stability-properties}.
\end{proof}
\vspace{-2mm}
\paragraph{Estimation of fractional Poincaré constants.} Based on the stability properties derived above, we provide in the next proposition an upper bound of the fractional Poincaré constant of heavy-tailed (S)GD for smooth and strongly convex losses.

\begin{restatable}{proposition}{propStronglyConvexStability}
    \label{prop:strongly-convec-poincaré}
    Assume that $\ell(w,z)$ is $\lambda$-strongly convex and $M$-smooth in $w$. We consider \Cref{eq:discrete-decursion-pure-jump} with $\mathcal{C} = \Rd$. We assume that the initial distribution of $X_0$ satisfies an $\alpha$-stable Poincaré inequality with constants $(\gamma, 0)$. We further assume that
    \begin{align}
        \label{eq:assumption-strongly-convex-smooth}
        \frac{\lambda}{M} \left( 1 + \frac{\alpha}{d} \right) > 1.
    \end{align}
    Let us define
    $c_0 := \frac{\eta \sigma^\alpha}{1 - F(\eta_0)}$ with $F(\eta) := \frac{(1 - \eta \lambda)^{\alpha + d}}{(1 - \eta M)^\alpha}$.
    Then $F(\eta_0) < 1$ and if $\gamma \leq c_0$, then the distribution of $X_k$ satisfies an $\alpha$-stable Poincaré inequality with constants $(c_0, 0)$ for all $k\in\mathds{N}$.
\end{restatable}

\begin{proof}
    See \Cref{sec:proofs-stability-properties}.
\end{proof}

In practice, this result imposes a relatively strong condition on the condition number $M / \lambda$, so that it would apply only for well-conditioned losses. This behavior is due to \Cref{lemma:bi-lipschitz-stability}, which introduces dimension-dependent quantities in the estimation of the $\alpha$-stable Poincaré inequality constants. 
To better understand this aspect, we observe that our \Cref{lemma:bi-lipschitz-stability} does not recover the Gaussian limit as $\alpha \to 2^-$, which could however be expected by the weighted BBM formula, \Cref{eq:weighted-BBM}. This shows that \Cref{lemma:bi-lipschitz-stability} might be largely improvable, leading to much better estimates of the fractional Poincaré constants. Improving such stability lemmas for fractional Poincaré inequalities is beyond the scope of this paper, and we leave it for future works.

\section{Conclusion}

In this work, we provided the first RDP guarantees for heavy-tailed SDEs with $\alpha$-stable noise.
We explored both the multifractal (\ie, with a non-trivial Gaussian component) and the pure-jump case.
In both cases, we obtain a semi-concentrated RDP guarantee with two regimes: \textrm{(i)} a time-uniform bound when under the assumption of an $\alpha$-stable Poincaré inequality and \textrm{(ii)} a regime where the bound can grow at most linearly with time but has the same dependence on the order of the Rényi divergence than in concentrated DP.
In both cases, we can convert our results to $(0,\delta)$-DP, in which case we obtain a dependence on the dimension comparable to existing works.
Finally, we extended our results in the discrete time setting and proved stability lemmas for $\alpha$-stable Poincaré inequalities, hence, providing theoretical foundations to satisfy these inequalities in practice.

\noindent \textbf{Limitations \& future works.} As already discussed above, some of our results rely on the existence of fractional Poincaré inequalities for the finite-time distributions generated by the heavy-tailed SDEs we consider.
We provided several stability lemmas for these inequalities, allowing to verify them in certain cases.
However, we believe that %
the estimate of the fractional Poincaré constants can be improved, and is an important (yet, difficult) direction for future works.
We observed in \Cref{sec:dp-of-heavy-tailed-sdes} that our main results have two regimes, one that can grow at most linearly with time but with a concentrated DP guarantee (\ie, $\landau{\beta}$), and one that is uniform in time but grows as $\beta^2$, under the assumption of a fractional Poincaré inequality. While this dependence on $\beta$ in the second case seems unavoidable with our approach, we believe that it is a relevant direction for future work, as it might require to strengthen certain functional inequalities existing in the literature.
Finally, relaxing or verifying the regularity conditions made in this work is another promising research direction, which might be related to recent fractional heat kernel estimates.

\vspace{-0.15in}
\acks{Jian Wang is supported by the
NSF of China the National Key R\&D Program of China (2022YFA1006003) and the National
Natural Science Foundation of China (Nos. 12225104 and 12531007).
Lingjiong Zhu is partially supported by the grants NSF DMS-2053454 and NSF DMS-2208303. Umut \c{S}im\c{s}ekli and Benjamin Dupuis' research is supported by European Research Council Starting Grant
DYNASTY-101039676.}

\newpage

\appendix

\paragraph{Organization of the appendix.} The appendix is organized as follows.

\begin{itemize}
    \item \Cref{sec:technical lemmas} presents three technical lemmas that are used in several proofs. 
    \item \Cref{sec:omitted-proofs} is dedicated to the detailed omitted proofs of our main results.
\end{itemize}

\section{Technical Lemmas}
\label{sec:technical lemmas}

In this section, we first prove two technical lemmas that are essential to obtain our main results. First, \Cref{lemma:magical-bregman-inequality} is a lower bound on certain Bregman divergences. The second result, \Cref{lemma:differential-exponential-inequality}, presents a solution to the differential inequality that naturally appears in our Rényi flow computations in \Cref{sec:renyi-flows}.

\begin{definition}[Bregman divergence]
    \label{def:bregman}
    Let $\Phi: %
    (0,\infty)\longrightarrow \R$ be a convex function. %
    Then we define the Bregman divergence, for $a, b \in (0,\infty)$, as
    \begin{align*}
        \bregman[\Phi]{a}{b} := \Phi(a) - \Phi(b) - \Phi'(b) (a - b).
    \end{align*}
    By convexity, it is clear that the above quantity is always non-negative. When $\Phi(x) = \Phi_\beta (x) := x^\beta$ for $\beta > 1$, we will shorten the notation and denote 
\begin{align}
    \label{eq:bregman-beta}
    \bregman[\beta]{a}{b} := \bregman[\Phi_\beta]{a}{b} = a^\beta + (\beta - 1) b^\beta - \beta a b^{\beta - 1}.
\end{align}
Note in particular that $\bregman{a}{b} := (a - b)^2$. 
\end{definition}

\begin{lemma}
    \label{lemma:magical-bregman-inequality}
    Let $\beta \geq 2$. %
    Then, for all $a,b \geq 0$:
    \begin{align*}
        \bregman{a^{\beta/2}}{b^{\beta/2}} \leq \bregman[\beta]{a}{b}.
    \end{align*}
\end{lemma}

\begin{proof}
    If $\beta = 2$, the result is obvious; if $\beta > 2$, we first note that %
    when $a=0$,
    \begin{align*}
        \bregman{a^{{\beta}/{2}}}{b^{{\beta}/{2}}} - \bregman[\beta]{a}{b} &= \left( a^{{\beta}/{2}} - b^{{\beta}/{2}} \right)^2 - \left( a^\beta + (\beta - 1) b^\beta - \beta a b^{\beta - 1} \right) = (2 - \beta) b^\beta \leq 0.
    \end{align*}
    Let us now assume $a > 0$. Then, we have
    \begin{align*}
    \bregman{a^{{\beta}/{2}}}{b^{{\beta}/{2}}} - \bregman[\beta]{a}{b} &= \left( a^{\beta/2} - b^{\beta/2} \right)^2 - \left( a^\beta + (\beta - 1) b^\beta - \beta a b^{\beta - 1} \right) \\
        &= -(\beta - 2) b^\beta - 2 (ab)^{{\beta}/{2}} + \beta a b^{\beta - 1} \\
        &= -(ab)^{{\beta}/{2}} \left( (\beta - 2) \left( \frac{b}{a} \right)^{{\beta}/{2}}  - \beta \left( \frac{b}{a} \right)^{{\beta}/{2} - 1} + 2 \right).
    \end{align*}
    Therefore, it makes sense to study the polynomial function 
    \begin{align*}
        P(x) = (\beta - 2) x^{\frac{\beta}{2}} - \beta x^{\frac{\beta}{2} - 1} + 2,\quad x\ge0.
    \end{align*}
    Clearly $P(0) = 2 > 0$. For any $x>0$, 
    \begin{align*}
        P'(x) = \frac{\beta(\beta - 2)}{2} x^{\frac{\beta}{2} - 2} (x - 1).
    \end{align*}
    Thus, the minimum of $P$ is $P(1) = 0$. The result follows. 
\end{proof}

In the next lemma, we solve a class of differential inequalities appearing in our theory.

\begin{lemma}
    \label{lemma:differential-exponential-inequality} 
    Consider a non-negative function $f: (0,\infty) \longrightarrow \R_+$ that is continuous and differentiable, %
    admits a right limit as $t\to 0^+$, and satisfies
    \begin{align}
        \label{eq:differential-inequality-lemma}
        f'(t) \leq K - a \left( 1 - e^{-f(t)} \right)\quad\text{for any $t>0$}, 
    \end{align} where $a,K > 0$ are two constants.
    Then, we have the following results:
    \begin{itemize}
        \item For any $t> 0$,  $f(t) \leq f(0^+) + Kt$.
        \item If $K < a$ and $f(0^+) \leq \log \left( \frac{a}{a-K} \right)$, then for any $t> 0$, $f(t) \leq \log \left( \frac{a}{a-K} \right)$.
        \item If $K < a$ and $f(0^+) > \log \left( \frac{a}{a-K} \right)$, then for any $t> 0$, %
       \begin{align*}
        f(t) \leq \log \left( \frac{a}{a-K} \right) + \log \left( 1 + e^{-(a - K) t}  \left( e^{f(0^+)} \frac{a - K}{a} - 1 \right) \right).
    \end{align*}

    \end{itemize}
\end{lemma}

\begin{proof}
    \textbf{Case $1$:} We note that for any $t>0$, $ f'(t) \leq K$. Therefore, for any $\varepsilon > 0$ %
    and any $t>0$, $ f(t) \leq f(\varepsilon) + K(t - \varepsilon)$. By taking the limit as $\varepsilon \to 0^+$, we conclude that $f(t) \leq Kt + f(0^+)$ for all $t\geq 0$.
    
    In the following, we denote $f_0 := \log\left( \frac{a}{a-K} \right)$.
    
    \textbf{Case $2$:} Assume that $K<a$ and $f(0^+) \leq f_0$. Suppose that there exists $\tau > 0$ such that $f(\tau) > f_0$.
    By \Cref{eq:differential-inequality-lemma}, this implies that $f'(\tau) < 0$.
    By continuity of $f$, there exists $t_0 \in (0,\tau]$ such that $f(t_0) = \sup_{s\in (0,\tau]} f(s)$. Necessarily, $t_0 > 0$, %
    and thus $f'(t_0) = 0$, as $f'(\tau) < 0$, which implies $f(t_0) \leq f_0$ by \Cref{eq:differential-inequality-lemma}; but that is absurd. The claim follows.

    \textbf{Case $3$:} Let $g(t) := f(t) - f_0$. Let $\tau := \inf\{t> 0: g(t) \leq 0\}$. By assumption, we know that $\tau > 0$. For $t \in (0,\tau)$, a simple computation shows that
    \begin{align*}
        \timeder \left( e^{g(t)} - 1 \right) = g'(t) e^{g(t)} \leq -(a - K) e^{g(t)} \left( 1 - e^{-g(t)} \right) = -(a - K) \left( e^{g(t)} - 1  \right).
    \end{align*}
    By Gr\"{o}nwall's inequality, we deduce that, for $t \in (0, \tau)$:
    \begin{align*}
        e^{g(t)} \leq 1 + e^{-(a - K) t}  \left( e^{g(0^+)} - 1 \right) = 1 + e^{-(a - K) t}  \left( e^{f(0^+)} \frac{a - K}{a} - 1 \right).
    \end{align*}
    Therefore,
    \begin{align*}
        f(t) \leq \log \left( \frac{a}{a-K} \right) + \log \left( 1 + e^{-(a - K) t}  \left( e^{f(0^+)} \frac{a - K}{a} - 1 \right) \right).
    \end{align*}
    For $t \ge \tau$, we can apply \textbf{Case $2$} and observe that the above inequality is still true; hence, it is proven for all $t > 0$.
\end{proof}
\vspace{-2mm}
\begin{remark}[Local version of \Cref{lemma:differential-exponential-inequality}]
    \label{rk:local-version-differential-inequality} \rm
    We observe that the proof \Cref{lemma:differential-exponential-inequality} still holds if we only have \Cref{eq:differential-inequality-lemma} for $t \in (0,T]$ for some $T>0$. In that case, all the claims are true for $t \in (0,T]$ instead of $t> 0$.
\end{remark}

The next lemma provides asymptotics for the constant $\calphad$ appearing in the infinitesimal generator of $\alpha$-stable L\'{e}vy processes (\Cref{eq:levy-generator}). Similar computations can be found in \citep{dupuis_generalization_2024}; we only sketch the proof for the sake of completeness.

\begin{lemma}
    \label{lemma:constants-asymptotics}
    We have the following results:
    \begin{align*}
        \calphad \underset{\alpha \to 2^-}{\sim} 
        (2 - \alpha) \pi^{-d / 2 } d \gammaof{\frac{d}{2}}  
        ,\quad \calphad \underset{d \to \infty}{\longrightarrow}  \frac{\alpha 2^{\alpha / 2 - 1} \pi^{-d / 2 }}{\gammaof{1 - \frac{\alpha}{2}}} \gammaof{\frac{d}{2}} d^{\alpha / 2} . 
    \end{align*}
\end{lemma}

\begin{proof}
   Recall that 
    $\calphad := \alpha 2^{\alpha - 1} \pi^{-d / 2 } \frac{\gammaof{\frac{\alpha + d}{2}}}{\gammaof{1 - \frac{\alpha}{2}}}$.
    By the continuity of the $\Gamma$-function, we have $\gammaof{\frac{\alpha + d}{2}} \to \gammaof{1 + \frac{d}{2}} = \frac{d}{2} \gammaof{\frac{d}{2}}$
    as $\alpha \to 2^-$. On the other hand, by Euler's reflection formula,  
    $\frac1{\gammaof{1 - \frac{\alpha}{2}}} = \frac{\sin \left( \frac{\pi \alpha}{2} \right) \gammaof{\frac{\alpha}{2}}}{\pi} \sim 1 - \frac{\alpha}{2}$
    as $\alpha \to 2^-$. Similarly, we can apply Gautschi's formula \citep{gautschi_elementary_1959} to obtain that
    $\calphad \sim \frac{\alpha 2^{\alpha - 1} \pi^{-d / 2 }}{\gammaof{1 - \frac{\alpha}{2}}} \gammaof{\frac{d}{2}} d^{\alpha / 2} 2^{-\alpha / 2}$
    as $d\to\infty$. The results follow immediately.
\end{proof}

\vspace{-6mm}
\section{Omitted Proofs of the Main Results}
\label{sec:omitted-proofs}

\subsection{Renyi flow along general Lévy processes}
\label{sec:renyi-flows-general-levy-process}

We first prove a more general version of \Cref{thm:renyi-flow-with-Dirichlet-forms} we allow the driving noise in \Cref{eq:mirror_sdes} to be an arbitrary Lévy process (with a symmetric Lévy measure).
This highlights one of the main advantages of our approach compared to the existing literature on differential privacy for heavy-tailed algorithms: our proof techniques are not restricted to $\alpha$-stable noise.

In this section, we therefore study general Lévy processes, as they have been introduced in \Cref{sec:levy-sdes-background}. 
Formally, it corresponds to replacing \Cref{eq:sde-intro} with the following SDE,
\begin{align}
    \label{eq:sde-general-levy}
    \der W_t = \nabla \er(W_t) \der t + \der L_t,
\end{align}
where $(L_t)_{t\geq 0}$ is a Lévy process with triplet given by $(0, \sqrt{2}I_d, \nu)$, with $\nu$ a \emph{symmetric} Lévy measure (\ie, $\nu(\der z) = \nu(-\der z)$).

Moreover, in this section, we also allow the drift terms in the SDEs to be time-dependent, which will be particularly helpful to extend our theory to the discrete-time setting in \Cref{sec:discrete-time}.

More precisely, we consider two probability densities following fractional Fokker-Planck equations (FPEs) \citep{duan_introduction_2015,umarov_beyond_2018} driven by different force fields. We define two time-dependent vector fields
\begin{align*}
    F_t, F_t': \Rd \longrightarrow \Rd,
\end{align*}
and we consider $p_t$ and $p_t'$ two probability densities that are smooth solutions of the following fractional FPEs, for some tail-index $\alpha \in (0,2)$:
\begin{equation}
    \label{eq:mirror_fpes_force_fields}\begin{cases}
       \partial_t p_t  = I[p_t] + \sigma_2^2 \Delta p_t + \nabla \cdot (p_t F_t),\\
        \partial_t p_t'  = I[p_t'] + \sigma_2^2 \Delta p_t' + \nabla \cdot (p_t' F_t'),
\end{cases}\end{equation}
where the pseudo-differential operator $A := \sigma_2 \Delta + \Irm$ corresponds to the infinitesimal generator of the Lévy process $(L_t)_{t\geq 0}$, according to \Cref{eq:levy-generator}.
Therefore, with the notations of \Cref{eq:levy-generator}, the non-local operator $\Irm$ is given for $u\in\cbtwo$ by
\begin{align*}
    I[u](x) := \intrdzero \left( u(x+z) - u(x) - \nabla u(x) \cdot z \chi(\normof{z}) \right) \der \nu(z).
\end{align*}
Note that the operator $A$ is self-adjoint with respect to the Lebesgue measure \citep{gentil_logarithmic_2008} (for smooth enough functions). We will use this fact repeatedly in our proofs.

This leads us to the following lemma, which is a Rényi flow computation along SDEs driven by general Lévy processes. 
This lemma fundamentally differs from the derivations of \citet{dupuis_generalization_2024} for two main reasons: (1) we focus on the Rényi flow instead of the entropy flow and (2) in this lemma we allow both dynamics to by time-varying, instead of considering stationary distributions.

\begin{restatable}{lemma}{lemmaRenyiFlow}
    \label{lemma:renyi-flow-two dynamics}
    Assume that \Cref{ass:regularity-conditions} holds. Then, for any $\beta>1$ and $t>0$,
    \begin{align*}
        \timeder \renyi{p_t}{p_t'} = -\frac{\sigma_\alpha^\alpha }{\beta - 1} \frac{\mathrm{B}_\beta^\alpha\left(p_t, p_t'\right)}{\ebeta \left(p_t, p_t'\right)} - \beta \sigma^2_2 \frac{\rinfo{p_t}{p_t'}}{\ebeta \left(p_t, p_t'\right)} + \beta \intrd v_t^{\beta - 1} \frac{  \langle \nabla v_t, F_t - F_t' \rangle}{\ebeta \left(p_t, p_t'\right)} p_t' \der x, 
    \end{align*}
    where  the term $B_\beta^\alpha(p_t, p_t')$ is called the Bregman integral and is defined by
    \begin{align*}
        \mathrm{B}^\alpha_\beta\left(p_t, p_t'\right) := C_{\alpha,d}\intrdzero \intrd \bregman[\beta]{\frac{p_t}{p_t'}(x)}{\frac{p_t}{p_t'}(x + z)} p_t'(x) \der x \der \nu(z).
    \end{align*}
\end{restatable}

\begin{proof}
    We consider the Fokker-Planck equations given by \Cref{eq:mirror_fpes_force_fields}, with the notations above. To simplify the notations, we use as above the operator
    $$
    A[u] := I[u] + \sigma_2 \Delta u.
    $$
    
    Let us denote $v_t := p_t / p_t'$ and $\Phi(x) := x^\beta$. \Cref{ass:regularity-conditions} ensures that we can differentiate under the integral, and write
    \begin{align*}
        (\beta - 1) \ebeta (p_t, p_t') \timeder \renyi{p_t}{p_t'} &= \timeder \intrd \Phi(v_t) p_t' \der x \\
        &= \intrd \Phi'(v_t) \left( \partial_t p_t - v_t \partial_t p_t' \right) \der x + \intrd \Phi(v_t) \partial_t p_t' \der x \\
        &= \intrd \Phi'(v_t) \left( A[p_t] + \nabla \cdot (p_t F_t) - v_t A[p_t'] - v_t \nabla \cdot \left(p_t' F_t'\right) \right) \der x \\& \quad   + \intrd \Phi(v_t) \left( A[p_t']  + \nabla \cdot \left(p_t' F_t'\right) \right) \der x.
    \end{align*}
    Using the facts that $I$ is self-adjoint in $L^2(\R^d;\der x)$ \citep{gentil_logarithmic_2008} and that $p_t$ and $p_t'$ vanish at infinity (so that the boundary terms are equal to $0$) by \Cref{ass:regularity-conditions}, we can integrate by parts to obtain that
    \begin{align*}
        (\beta - 1) \ebeta (p_t, p_t') \timeder \renyi{p_t}{p_t'} &= C_{\mathrm{diffusion}} + C_{\mathrm{fractional}} + C_{\mathrm{potential}},
    \end{align*}
    where
    \begin{align*}
        C_{\mathrm{diffusion}} :
        &= \sigma_2^2 \intrd p_t' \left( v_t \Delta \Phi'(v_t) - \Delta (v_t \Phi'(v_t)) + \Delta \Phi(v_t) )  \right) \der x \\
        &= \sigma_2^2 \intrd p_t' \Big( v_t \nabla \cdot (\Phi''(v_t) \nabla v_t) 
        \\
        &\qquad\qquad\qquad-  \nabla \cdot (  \Phi'(v_t) \nabla v_t + v_t \Phi''(v_t) \nabla v_t ) +  \nabla \cdot (\Phi'(v_t) \nabla v_t) )  \Big) \der x \\
        &= \sigma_2^2 \intrd p_t' \left( v_t \nabla \cdot \left(\Phi''(v_t) \nabla v_t\right) -  \nabla \cdot \left(   v_t \Phi''(v_t) \nabla v_t \right)  \right) \der x \\
        &= -\sigma_2^2 \intrd \Phi''(v_t) \normof{\nabla v_t}^2 p_t' \der x  = -\sigma_2^2 \beta (\beta - 1) \rinfo{p_t}{p_t'} ,\\
            C_{\mathrm{potential}}: &= -\intrd p_t' \left( v_t \Phi''(v_t) \left\langle \nabla v_t, F_t \right\rangle + \left\langle \nabla (v_t \Phi'(v_t))), F_t' \right\rangle - \Phi'(v_t) \left\langle \nabla v_t, F_t'  \right\rangle \right) \der x\\
         &= -\intrd p_t' \left( v_t \Phi''(v_t) \left\langle \nabla v_t, F_t \right\rangle + v_t \Phi''(v_t)\left\langle \nabla v_t , F_t' \right\rangle  \right) \der x\\
         &= \intrd v_t \Phi''(v_t) \left\langle \nabla v_t, F_t'- F_t \right\rangle p_t' \der x = \beta (\beta - 1) \intrd v_t^{\beta - 1} \left\langle \nabla v_t, F_t'- F_t \right\rangle p_t' \der x
    \end{align*}
    and
    \begin{align*}
         C_{\mathrm{fractional}} := \intrd \left( I [\Phi'(v_t)] v_t - I [v_t \Phi'(v_t)] + I[\Phi(v_t)] \right) p_t' \der x.
    \end{align*}
    
    Now we use the expression of the non-local operator as 
    \begin{align*}
        I [\phi](x)  =  \intrdzero \left( \phi(x + z) - \phi(x) - \langle\nabla \phi (x), z \rangle \chi(\normof{z}) \right) \der \nu(z).
    \end{align*}
    The terms containing $\chi(\normof{z})$ cancel each other out, so that we obtain 
    \begin{align*}
        C_{\mathrm{fractional}} &=  \intrd \intrdzero \Big\{v_t(x) \left[\Phi'(v_t(x+z)) - \Phi'(v_t(x)) \right] - v_t(x+z) \Phi'(v_t(x+z)) \\&\quad \quad \quad \quad\qquad\qquad + v_t(x) \Phi'(v_t(x)) + \Phi(v_t(x+z)) - \Phi(v_t(x)) \Big\} p_t'(x)\der \nu(z) \der x.
    \end{align*}
    After rearranging the terms, we obtain the following expression
    \begin{align*}
    C_{\mathrm{fractional}} = - C_{\alpha,d} \intrd\intrdzero \bregmanphi{v_t(x)}{v_t(x+z)} p_t'(x) \der \nu(z) \der x.
    \end{align*}
    The proof is complete.
\end{proof}
\begin{remark}\rm
    We observe that the proof does not make use of the particular form of the Lévy measure of the rotationally invariant $\alpha$-stable L\'{e}vy process. Indeed, the proof is valid for any %
    symmetric %
    Lévy process as the noise driving the two SDEs.
\end{remark}

\subsection{Omitted proofs of \Cref{sec:renyi-flows}}
\label{sec:omitted-proofs-renyi-flows}

We are now ready to present the proof of \Cref{thm:renyi-flow-with-Dirichlet-forms}, which is a Rényi flow computation along heavy-tailed (multifractal) SDEs.

\paragraph{Proof of \Cref{thm:renyi-flow-with-Dirichlet-forms}.} 

    As before, we denote $v_t := p_t / p_t'$.
    We apply \Cref{lemma:renyi-flow-two dynamics} when $F_t = \nabla \er$ and $F_t'= \nabla\erprime$. Moreover, we chose the Lévy process $(L_t)_{t\geq 0}$ in \Cref{eq:mirror_fpes} to be the $\alpha$-stable Lévy process, as described in \Cref{sec:levy-sdes-background}. 
    
    With these notations, we have that
    \begin{align*}
        \mathrm{B}_\beta^\alpha (p_t, p_t') &= C_{\alpha,d} \intrdzero\intrd \bregman[\beta]{v_t(x)}{v_t(x + z)} p_t'(x) \frac{\der x \der z}{\normof{z}^{d + \alpha}} \\
        &\geq C_{\alpha,d} \intrdzero\intrd \bregman[2]{v_t^{\beta/2}(x)}{v_t^{\beta/2}(x + z)} p_t'(x) \frac{\der x \der z}{\normof{z}^{d + \alpha}} \by{\Cref{lemma:magical-bregman-inequality}} \\
        &= 2\ecal_{\alpha, p_t'}\left(v_t^{\beta / 2}, v_t^{\beta / 2}\right).
    \end{align*}
    We also have
    \begin{align*}
        \rinfo{p_t}{p_t'} = \intrd v_t^{\beta - 2} \normof{\nabla v_t}^2 \der x = \frac{4}{\beta^2}\intrd \normof{\nabla v_t^{\beta / 2}}^2 \der x = \frac{4}{\beta^2} \ecal_{2, p_t'}\left(v_t^{\beta / 2}, v_t^{\beta / 2}\right).
    \end{align*}
    The proof then follows from \Cref{lemma:renyi-flow-two dynamics}.
    \hfill $\blacksquare$

We conclude this section by presenting the proof of \Cref{lemma:bregman-integral-lower-bound}.

\paragraph{Proof of \Cref{lemma:bregman-integral-lower-bound}. }
    Let us introduce the notation:
    \begin{align*}
        \mathcal{J} := \frac{2\sigma_\alpha^\alpha}{\beta - 1} \ecal_{\alpha, p_t'}\left(v_t^{\beta / 2}, v_t^{\beta/2}\right) + \frac{2 \sigma_2^2}{\beta} \ecal_{2, p_t'}\left(v_t^{\beta / 2}, v_t^{\beta/2}\right),
    \end{align*}
    with $v_t := p_t / p_t'$.
    Let us denote $u_t := v_t^{\beta/2}$.
    By the $\alpha$-stable Poincaré inequality, we have
    \begin{align*}
        \mathcal{J} &= \frac{\sigma_\alpha^\alpha C_{\alpha,d}}{\beta - 1} \intrdzero\intrd (u_t(x) - u_t(x+z))^2 p_t'(x) \frac{\der x \der z}{\normof{z}^{d + \alpha}}  + \frac{2 \sigma_2^2}{\beta} \intrd \normof{\nabla u_t}^2 \der x \\
        &\geq \frac{\sigma_\alpha^\alpha C_{\alpha,d}}{\beta} \intrdzero\intrd (u_t(x) - u_t(x+z))^2 p_t'(x) \frac{\der x \der z}{\normof{z}^{d + \alpha}}  + \frac{ \sigma_2^2}{\beta} \intrd \normof{\nabla u_t}^2 \der x \\
        &\geq \frac{1}{\gamma \beta } \left\{ \intrd \left(\frac{p_t}{p_t'}\right)^{\beta} p_t' \der x - \left( \intrd \left(\frac{p_t}{p_t'}\right)^{\beta/2} p_t' \der x \right)^2 \right\} \\
         &= \frac{1}{\gamma\beta}  \left\{ e^{(\beta - 1)\renyi[\beta]{p_t}{p_t'}} - e^{(\beta - 2)\renyi[\beta/2]{p_t}{p_t'}} \right\} \\
          &\geq \frac{1}{\gamma \beta} \left\{ e^{(\beta - 1)\renyi[\beta]{p_t}{p_t'}} - e^{(\beta - 2)\renyi[\beta]{p_t}{p_t'}} \right\} \by{$\beta \mapsto \mathrm{R}_\beta$ is non-decreasing} \\
          &= \frac{1}{\gamma \beta} \ebeta (p_t, p_t') \left( 1 - e^{-\renyi[\beta]{p_t}{p_t'}} \right).
    \end{align*}
    This completes the proof of the first assertion. For the second one, one can follow the same arguments as above. 
    \hfill $\blacksquare$

\subsection{Omitted proofs of \Cref{sec:pure-jump-case}}
\label{sec:omitted-proofs-pure-jump-jv-asssumptions}

The following lemma is a formal proof of \Cref{eq:weighted-BBM} and serves as a motivation for our approach in \Cref{sec:pure-jump-case}. It can be seen as a weighted version of the celebrated Bourgain-Brezis-Mironescu theorem (BBM formula), under stronger assumptions.
Note that a similar observation can be found in \citep[Proposition B.1]{he_separation_2024-1}.

\begin{restatable}[The weighted BBM formula]{lemma}{lemmaWetightedBBM}
    \label{lemma:weighted-BBM}
    Let $f \in \cbtwo$ and $\mu$ any Borel probability measure on $\Rd$. We have
    \begin{align*}
        \ecal_{\alpha, \mu}(f,f) \underset{\alpha \to 2^-}{\longrightarrow}  \ecal_{2, \mu}(f,f).
    \end{align*}
\end{restatable}

\begin{proof}
    We write $\ecal_{\alpha, \mu}(f,f) = \frac{C_{d,\alpha}}{2}\left(I_1 + I_2\right)$, where 
    \begin{align*}
        I_1 :=  \intrd \int_{\normof{z} \leq 1} \frac{(f(x+z) - f(x))^2}{\normof{z}^{d + \alpha}} \der \mu(x) \der z,
    \end{align*} and 
    \begin{align*}
        I_2 := \intrd \int_{\normof{z} > 1} \frac{(f(x+z) - f(x))^2}{\normof{z}^{d + \alpha}} \der \mu(x) \der z \leq 2 \cbnorm{f}^2 \int_{\normof{z} > 1} \frac{\der z} {\normof{z}^{d + \alpha}} = \frac{C}{\alpha}.
    \end{align*}
    Here, $C< \infty$ is a constant independent of $\alpha$. 
    On the other hand, by the Taylor-Lagrange formula, for every $x,z \in \Rd$,  $f(x+z) - f(x)= \langle z, \nabla f(x) \rangle + \langle z, \nabla^2f(\xi_{x,z}) z \rangle / 2$, with $\xi_{x,z} \in  \{ x + a z,~ a \in [0,1] \}$. Therefore, we can further write $I_1 = I_3 + I_4$ as
    \begin{align*}
        I_3 :&= \intrd \int_{\normof{z} \leq 1} \frac{ \frac1{4}\left\langle z, \nabla^2f(\xi_{x,z}) z \right\rangle^2 + \left\langle z, \nabla^2f(\xi_{x,z}) z \right\rangle \left\langle z, \nabla f(x) \right\rangle }{\normof{z}^{d + \alpha}} \der \mu(x) \der z \\
        &\lesssim \cbnorm{f}^2 \int_{\normof{z} \leq 1} \frac{\der z}{\normof{z}^{d+\alpha - 3}} = \frac{C'}{3 - \alpha}
    \end{align*}
    with $C'< \infty$ being another constant independent of $\alpha$, and
    \begin{align*}
        I_4: = \intrd \int_{\normof{z} \leq 1} \frac{\langle z, \nabla f (x) \rangle^2}{\normof{z}^{d+\alpha}} \der \mu(x) \der z = \frac{\mathrm{Vol}(\Sd)}{d} \intrd \int_0^1 \normof{\nabla f(x)}^2 \der \mu(x) \frac{\der r}{r^{\alpha - 1}}, 
    \end{align*} thanks to %
    a spherical change of variable and the invariance by rotation.
    
    Therefore, by \Cref{lemma:constants-asymptotics}, we have
    \begin{align*}
        \ecal_{\alpha, \mu}(f,f) \underset{\alpha \to 2^-}{\longrightarrow} \frac1{2} (2 - \alpha) \pi^{-d / 2 } d \gammaof{\frac{d}{2}} \, \frac{2 \pi^{d / 2}}{d \gammaof{\frac{d}{2}}}  \frac1{2 - \alpha} \intrd \normof{\nabla f(x)}^2 \der \mu(x) =  \ecal_{2,\mu}(f,f).
    \end{align*}
    This completes the proof.
\end{proof}

We can now discuss the proof of \Cref{thm:dp_guarantees_sdes_pure_jump}, which is based on two technical lemmas. 
The next lemma is an integral representation of the Dirichlet form $\ecal_{\alpha,\mu}$ appearing in our Rényi flow computations.

\begin{lemma}[Spherical parameterization of the Dirichlet form $ \ecal_{\alpha, \mu}(f,f)$]
    \label{lemma:spherical-parmaterization}
    Consider $u : \Rd \longrightarrow \R$ a $\mathcal{C}^1$ function with bounded gradient. Let $\mu \in \mathcal{P}(\Rd)$. Then, 
    \begin{align*}
       \ecal_{\alpha, \mu} (u,u) = \frac{\calphad \sigma_{d - 1}}{2d} \int_0^{+\infty} \mathscr{J} (r, u, \mu) \frac{\der r}{r^{\alpha - 1}},
    \end{align*}
    where $\sigma_{d - 1} := \mathrm{Vol}(\Sd)=\frac{2 \pi^{d/2}}{\gammaof{\frac{d}{2}}} $ and $\mathscr{J} (\cdot, u, \mu) : \R_+ \to \R_+$ is a continuous function satisfying
    \begin{align*}
        \mathscr{J} (0, u, \mu) =  \ecal_{2, \mu} (u,u) = \normof{\nabla u}^2_{L^2(\mu)}.
    \end{align*}
\end{lemma}

\begin{proof}
We apply a spherical change of coordinates to get
    \begin{align*}
        2\ecal_{\alpha, \mu}\left(u, u\right) &= \calphad\intrd \intrd (u(x+ z) - u(x))^2   \frac{\der \mu(x) \der z }{\normof{z}^{d + \alpha}} \\
        & = \calphad \int_0^{+\infty} \intsd \intrd \left( \int_0^r \langle \theta, \nabla u(x + s\theta) \rangle \der s\right)^2  \der \mu(x) \der\sigma (\theta) \frac{\der r}{r^{1 + \alpha}} ,
    \end{align*}
        where $\sigma$ is the Hausdorff measure on the sphere $\Sd$. By Tonelli's theorem, we can rewrite the above integral as
    \begin{align*}
         2\ecal_{\alpha, \mu}\left(u, u\right) = \kalphad \int_0^{+\infty} \mathscr{J} (r, u, \mu) \frac{\der r}{r^{\alpha - 1}},
    \end{align*}
   where the function $\mathscr{J}(\cdot, u, \mu)$ defined for $r> 0$ by
    \begin{align}\label{e:int}
        \mathscr{J}(r,u, \mu) := \frac{d}{ \sigma_{d - 1}} \int_{\Rd \times \Sd}  \left( \frac1{r} \int_0^r \langle \theta, \nabla u(x + s\theta) \rangle \der s \right)^2 \der (\mu \otimes \sigma) (x, \theta)
    \end{align} with  $\kalphad := \frac{\calphad \sigma_{d - 1}}{d}.$
 We note that the integral above is clearly finite under our assumptions on $u$. Moreover, we note that for all $(x, \theta) \in \Rd \times \Sd$, we have by continuity of $\nabla u$ that
    \begin{align*}
        \frac1{r} \int_0^r \langle \theta, \nabla u(x + s\theta) \rangle \der s \underset{r \to 0^+}{\longrightarrow} \langle \theta, \nabla u (x) \rangle .  
    \end{align*}
    Moreover, by the Cauchy-Schwarz inequality,
    \begin{align*}
        \left( \frac1{r} \int_0^r \langle \theta, \nabla u(x + s\theta) \rangle \der s\right)^2 \leq \normof{\nabla u}^2_{\infty} < +\infty.
    \end{align*}
    Therefore, by the dominated convergence theorem with respect to $\mu \otimes \sigma$, 
    \begin{align*}
         \mathscr{J}(r,u,\mu) \underset{r \to 0^+}{\longrightarrow}  \frac{d}{ \sigma_{d - 1}}  \int_{\Rd \times \Sd}    \langle \theta, \nabla u (x) \rangle^2  \der (\mu \otimes \sigma) (x, \theta).
    \end{align*}
    By invariance by rotation and Tonelli's theorem, we deduce that
    \begin{align*}
        \mathscr{J}(r,u,\mu) \underset{r \to 0^+}{\longrightarrow}  \frac{d}{ \sigma_{d - 1}}  \intrd \normof{\nabla u (x)}^2 \der \mu(x) \intsd \langle \theta, e_1 \rangle^2  \der \sigma (\theta) = \normof{\nabla u}^2_{L^2(\mu)},
    \end{align*}
    where $e_1 \in \Sd$ is arbitrary. Therefore, we can continuously extend the function $\mathscr{J}(\cdot, u, \mu)$ on $\R_+$ by setting $\mathscr{J}(0,u, \mu) = \normof{\nabla u}^2_{L^2(\mu)} = \ecal_{2,\mu} (u, u) $. By the dominated convergence, we also obtain that $\mathscr{J}(\cdot, u, \mu)$ is continuous on  $\R_+$. This concludes the proof.
\end{proof}

\begin{lemma}[Lower bound on the Dirichlet form]
    \label{lemma:dirichlet-lower-bound-pure-jump-case}
    Assume that $\beta\ge2$, and that $v_t$ is twice continuously differentiable for all $t \in(0, T]$ so that 
    \begin{align}  
        \sup_{S \simeq S'} \sup_{t\leq T} \left( %
        \normof{\nabla v_t^{\beta / 2}}_\infty + \normof{\nabla^2 v_t^{\beta / 2}}_\infty \right) < +\infty, \quad  %
        \inf_{S \simeq S'  } \inf_{t \leq T } \normof{\nabla v_t^{\beta / 2}}_{L^2(p_t')} >  0,
    \end{align}
    where $S\simeq S'$ ranges over all neighboring datasets in $\bigcup_{n \geq 1} \zcal^n$. Then there exists a constant $R > 0$ (that may depend only on $d$ and $\beta$ as well as $T$) such that for all $t\in [0, T]$
    \begin{align*}
      2\ecal_{\alpha, p_t'} \left(v_t^{\beta / 2}, v_t^{\beta / 2}\right) \geq   \frac{\calphad \sigma_{d - 1}}{2d (2 - \alpha)} R^{2 - \alpha} \ecal_{2,p_t'}\left(v_t^{\beta / 2}, v_t^{\beta / 2}\right). 
    \end{align*}
\end{lemma}

\begin{proof}
    fix $\beta > 0$.
    Let us %
    introduce $u_t := v_t^{\beta / 2}$ for all $t\in [0,T]$. By our assumptions, 
    \Cref{lemma:spherical-parmaterization} gives us the spherical representation of the Dirichlet form $  \ecal_{\alpha, p_t'}$:
    \begin{align*}
       \ecal_{\alpha, p_t'} (u_t,u_t) = \frac{\calphad \sigma_{d - 1}}{2d} \int_0^{+\infty} \mathscr{J} (r, u_t, p_t') \frac{\der r}{r^{\alpha - 1}}.
    \end{align*}
  Let us denote the integrand in $\mathscr{J}(\cdot, u_t, p_t')$ (see \eqref{e:int}) by
    \begin{align*}
        F(\theta, x, r) := \left( \frac1{r} \int_0^r \langle \theta, \nabla u (x + s \theta) \rangle\der s  \right)^2.
    \end{align*}
    We have 
    \begin{align*}
        \left| \frac{\partial F}{\partial r} \right| &= 2\left| \frac1{r} \int_0^r \langle \theta, \nabla u (x + s \theta) \rangle\der s  \left( \frac{-1}{r^2}  \int_0^r \langle \theta, \nabla u (x + s \theta) \rangle\der s + \frac1{r^2} \int_0^r \langle \theta, \nabla u (x + r \theta)\rangle  \der s \right) \right| \\
      & \leq  2 \normof{\nabla u}_{\infty}  \left| \frac1{r^2} \int_0^r \int_s^r \left\langle \theta, \nabla^2 u (x + l \theta) \theta\right\rangle \der l  \der s  \right| \\
       &\leq \normof{\nabla u}_{\infty}\normof{\nabla^2 u}_{\infty} < +\infty.
    \end{align*}
    Therefore, we can differentiate under the sum and use the expression of $\mathscr{J}(\cdot, u_t, p_t')$ (given in \Cref{lemma:spherical-parmaterization}) to obtain that
    \begin{align*}
        \left| \frac{\partial \mathscr{J}}{\partial r} (r, u_t, p_t') \right| \leq C_1 := d \normof{\nabla u}_{\infty}\normof{\nabla^2 u}_{\infty} < \infty.
    \end{align*} In particular, for all $r>0$ and $t\in [0,T]$,
    $$|\mathscr{J}(r, u_t, p_t')- \mathscr{J}(0, u_t, p_t')|\le C_1r.$$
    On the other hand, by our assumptions, we have $C_2 :=  \inf_{S \simeq S'  } \inf_{t \leq T } \normof{\nabla u_t}^2_{L^2(p_t')} >  0$. Thus, for all $t\in [0,T],$
    $$\mathscr{J}(0, u_t, p_t')\ge C_2.$$
       Therefore, %
   there exists $R>0$ (potentially depending only on $\beta$ and $d$ as well as $T$) such that for all $r \leq R$,  $\mathscr{J}(r, u_t, p_t') \geq (1/2) \mathscr{J}(0, u_t, p_t') $. Thus, by \Cref{lemma:spherical-parmaterization}, for all $r \leq R$, we have
         \begin{align*}
       2\ecal_{\alpha, p_t'} (u_t,u_t) &\geq \frac{\calphad \sigma_{d - 1}}{d} \int_0^{R} \mathscr{J} (r, u_t, p_t') \frac{\der r}{r^{\alpha - 1}} \\
       &=  \frac{\calphad \sigma_{d - 1}}{2d} \ecal_{2,p_t'}\left(v_t^{\beta / 2}, v_t^{\beta / 2}\right) \int_0^{R} \frac{\der r}{r^{\alpha - 1}}=  \frac{\calphad \sigma_{d - 1}}{2d (2 - \alpha)} R^{2 - \alpha} \ecal_{2,p_t'}\left(v_t^{\beta / 2}, v_t^{\beta / 2}\right). 
    \end{align*}
    This completes the proof.
\end{proof}

We now present the proof of our RDP guarantees for SDEs driven by pure-jump Lévy noise.

\paragraph{Proof of \Cref{thm:dp_guarantees_sdes_pure_jump}.}

    We start the proof similarly to the proof of \Cref{thm:multifractal-case-finite-sensitivity}, with the difference that now $\sigma_2 = 0$ (recall that we denote $\sigma :=\sigma_\alpha$ to simplify the notation).
    By \Cref{thm:renyi-flow-with-Dirichlet-forms}, for $t> 0$, 
     \begin{align*}
        \timeder \renyi{p_t}{p_t'} \leq -\frac{2\sigma^\alpha }{\beta - 1} \frac{\ecal_{\alpha, p_t'}\left(v_t^{\beta / 2}, v_t^{\beta/2}\right)}{\ebeta (p_t, p_t')}  + \mathrm{R}_{\mathrm{potential}}. 
    \end{align*}
    We split the first term on the right-hand side into two halves and use \Cref{lemma:bregman-integral-lower-bound} to obtain
    \begin{align*}
        \timeder \renyi{p_t}{p_t'} \leq  -\frac{1}{2\gamma (\beta - 1)} \left( 1 - e^{-\renyi[\beta]{p_t}{p_t'}} \right) - \frac{\sigma^\alpha }{(\beta - 1)} \frac{\ecal_{\alpha, p_t'}\left(v_t^{\beta / 2}, v_t^{\beta/2}\right)}{\ebeta (p_t, p_t')} + \mathrm{R}_{\mathrm{potential}}.
    \end{align*}
    We can now apply \Cref{lemma:dirichlet-lower-bound-pure-jump-case}, which gives that
    \begin{align*}
        \timeder \renyi{p_t}{p_t'}& \leq  - \frac{1 - e^{-\renyi[\beta]{p_t}{p_t'}} }{2\gamma (\beta - 1)}  - \frac{\sigma^\alpha \calphad \sigma_{d - 1} R^{2 - \alpha} }{4 d (\beta - 1) (2 - \alpha)} \frac{\ecal_{2, p_t'}\left(v_t^{\beta / 2}, v_t^{\beta/2}\right)}{\ebeta (p_t, p_t')}  + \mathrm{R}_{\mathrm{potential}}.
    \end{align*}
    By the Cauchy-Schwarz and Young inequalities, we have for all $\lambda > 0$ that
    \begin{align*}
        \mathrm{R}_{\mathrm{potential}} &=  \beta \intrd v_t^{\beta - 1} \frac{ \left\langle \nabla v_t, \nabla \widehat{\mathcal{R}}_{S'} - \nabla \er \right\rangle }{\ebeta (p_t, p_t')} p_t' \der x \\
        &\leq \frac{\beta}{\ebeta (p_t, p_t')} \left( \frac{\lambda}{2} \intrd v_t^{\beta - 2} \normof{\nabla v_t}^2 p_t'\der x + \frac{1}{2 \lambda} \intrd \normof{\nabla \er - \nabla \erprime}^2 v_t^\beta p_t' \der x  \right) \\
        &\leq \frac{2 \lambda}{\beta} \frac{\ecal_{2, p_t'}\left(v_t^{\beta / 2}, v_t^{\beta/2}\right)}{\ebeta (p_t, p_t')} + \frac{ \beta \sensitivity_g^2}{2 \lambda n^2}.
    \end{align*}
    Therefore, we make the following choice for the parameter $\lambda
        := \frac{\sigma^\alpha \calphad \sigma_{d-1}}{8d (\beta-1) (2 - \alpha)} R^{2 - \alpha}$.
    This implies that
    \begin{align*}
        \timeder \renyi{p_t}{p_t'} &\leq  -\frac{1}{2 \gamma (\beta - 1)} \left( 1 - e^{-\renyi[\beta]{p_t}{p_t'}} \right)  + \frac{4 \sensitivity_g^2 d (2 - \alpha) (\beta - 1)}{n^2 \sigma^\alpha \calphad \sigma_{d-1}} \\
        &= -\frac{1}{2 \gamma (\beta - 1)} \left( 1 - e^{-\renyi[\beta]{ p_t}{p_t'}} \right)  +  \frac{4  \sensitivity_g^2 (2 - \alpha) d \gammaof{\frac{d}{2}} \gammaof{1 - \frac{\alpha}{2}} (\beta - 1)}{ \alpha 2^{\alpha}\sigma^\alpha  R^{2 - \alpha} n^2 \gammaof{\frac{d + \alpha}{2}}} .
    \end{align*}
    We introduce %
    $a := \frac{1}{2 \gamma (\beta - 1)}$ and $K_n := \frac{\kalphad (\beta - 1) \sensitivity_g^2}{\sigma^\alpha n^2}$,
    with
    $\kalphad := \frac{4 (2 - \alpha) d \gammaof{\frac{d}{2}} \gammaof{1 - \frac{\alpha}{2}}}{ \alpha 2^{\alpha} R^{2 - \alpha}\gammaof{\frac{d + \alpha}{2}}}$.
       Thus, 
    \begin{align}
        \label{eq:proof-pure-jump-differential-inequality}
        \timeder \renyi{p_t}{p_t'} \leq K_n - a \left( 1 - e^{-\renyi[\beta]{p_t}{p_t'}} \right).
    \end{align}
    We observe that \Cref{ass:regularity-conditions} implies the continuity of $t \mapsto \renyi{p_t}{p_t'} $ at $t=0$, by the dominated convergence theorem. 
    The claim follows immediately from the local version of \Cref{lemma:differential-exponential-inequality} (see \Cref{rk:local-version-differential-inequality}) by noting that both SDEs are initialized with the same probability distributions.
    \hfill $\blacksquare$

\subsection{Omitted proofs of \Cref{sec:dp-heavy-tailed-gradient-descent}}
\label{sec:proofs-discrete-pure-jump}

We give below the proof of \Cref{thm:discrete-pure-jump} that provides RDP guarantees for discretization of pure-jump SDEs. 

\paragraph{Proof of \Cref{thm:discrete-pure-jump}.}

For any $t>0$ and $\theta\in \R^d$, denote $\mathfrak{D}_t (\theta) = V_k(\theta, t) - V_k'(\theta, t)$. Let us fix $k \in \mathds{N}$ and consider $t_k < t < t_{k+1}$. By utilizing the orthogonal projection and \Cref{ass:finite-sensitivity-on-convex}, we observe that, for all $\theta \in \Rd$,
    \begin{align*}
    \normof{\mathfrak{D}_t(\theta)}^2 = \frac1{n^2} \normof{\Eof{F_2^{(k)}(\Theta_{t_k}) \big| \Theta_t = \theta} + \Eof{F_2^{(k)}(\Theta'_{t_k}) \big| \Theta'_t = \theta}}^2.
    \end{align*}
    Recall that the neighboring datasets $S\simeq S' \in \zcal^n$ are fixed. Let $i_0 \in \{ 1,\dots, n\}$ be the (fixed) only index on which both datasets differ (\ie, $z_{i_0} \neq z_{i_0}'$). As defined in \Cref{sec:discrete-time}, the random batches $\Omega_k$ are independent of the Lévy processes appearing in \Cref{eq:chourasia-recursions}. Therefore, we can condition on the even that $i_0 \in \Omega_k$, and write that
    \begin{align*}
        \normof{\Eof{F_2^{(k)}(\Theta_{t_k}) \big| \Theta_t = \theta}} = \Pof{i_0 \in \Omega_k} \normof{\Eof{F_2^{(k)}(\Theta_{t_k}) \big| \Theta_t = \theta,~ i_0 \in \Omega_k}} \leq \frac{b}{n} \cdot \frac{\mathcal{S}_{g,\mathcal{C}}}{2b} = \frac{\mathcal{S}_{g,\mathcal{C}}}{2n},
    \end{align*}
    and similarly for $\Eof{F_2^{(k)}(\Theta'_{t_k}) \big| \Theta'_t = \theta}$. This gives us that
    \begin{align*}
    \normof{\mathfrak{D}_t(\theta)}^2 \leq \frac{ \mathcal{S}_{g,\mathcal{C}}^2}{n^2}.
    \end{align*}
      Therefore, we can replicate the lines of the proof of \Cref{thm:dp_guarantees_sdes_pure_jump} for $t_k < t < t_{k+1}$, and the only difference is that we now have, for all $\lambda > 0$,
    \begin{align*}
        \mathrm{R}_{\mathrm{potential}} \leq \frac{2 \lambda}{\beta} \frac{\ecal_{2, p_t'}\left(v_t^{\beta / 2}, v_t^{\beta/2}\right)}{\ebeta (p_t, p_t')} + \frac{ \beta }{2 \lambda} \sup_{\theta\in\Rd} \normof{\mathfrak{D}_t(\theta)}^2 \leq \frac{2 \lambda}{\beta} \frac{\ecal_{2, p_t'}\left(v_t^{\beta / 2}, v_t^{\beta/2}\right)}{\ebeta (p_t, p_t')} + \frac{ \beta     \mathcal{S}_{g,\mathcal{C}}^2 }{ \lambda n^2}.
    \end{align*}    
    Therefore, we obtain by the same lines as the proof of \Cref{thm:dp_guarantees_sdes_pure_jump} an equation similar to \Cref{eq:proof-pure-jump-differential-inequality} (formally, $    \mathcal{S}_{g,\mathcal{C}}$ plays the same role as $\sensitivity_g$). Thus, we have
    \begin{align*}
        \timeder \renyi{p_t}{p_t'} \leq K_n - a \left( 1 - e^{-\renyi[\beta]{p_t}{p_t'}} \right), \quad t_k < t < t_{k+1}
    \end{align*}
    with the quantities $a$ and $K_n$ defined in \Cref{thm:discrete-pure-jump}.
    
    \textbf{Case $1$:} By the local version of \Cref{lemma:differential-exponential-inequality} (see \Cref{rk:local-version-differential-inequality}), we obtain
    \begin{align*}
        \renyi{p_t}{p_t'} \leq \lim_{s \to t_k^+} \renyi{p_s}{p_s'} + K_n ( t - t_k), \quad t_k < t < t_{k+1}.
    \end{align*}
    By taking the limit as $t \to {t_{k+1}^-}$ and applying the (assumed) right-continuity of the Rényi divergences, we have
    \begin{align*}
        \lim_{s \to t_{k+1}^-} \renyi{p_s}{p_s'} \leq  \renyi{p_{t_k^+}}{p_{t_k^+}'} + K_n\eta,
    \end{align*}
    where $p_{t_k^-}$ denotes the distribution of $\Theta_{t_k}^-$.
    Now we note that as $t\to t_k^+$, both $\Theta_t$ and $\Theta_t'$ are subjected to the same mapping; hence, by the data processing inequality and the lower semicontinuity of the Rényi divergence, we have
    \begin{align*}
        \lim_{s \to t_{k+1}^-} \renyi{p_s}{p_s'} \leq \lim_{s \to t_k^-} \renyi{p_s}{p_s'} + K_n\eta.
    \end{align*}
    By recursion, we deduce that for all $N \in \mathds{N}$,
    \begin{align*}
        \renyi{p_{N\eta}}{p_{N\eta}'} \leq  \frac{\kalphad (\beta-1) \mathcal{S}_{g,\mathcal{C}}^2}{\sigma^\alpha n^2} N \eta.
    \end{align*}

    \textbf{Case $2$:} $K_n < a$. We show by recursion on $k \in \mathds{N}$ that
    \begin{align}
        \label{eq:recursion-hypothesis}
        \lim_{s \to t_{k}^-} \renyi{p_s}{p_s'} \leq \log \left( \frac{a}{a - K_n} \right).
    \end{align}
    For $k=0$, we have $\renyi{p_s}{p_s'} = 0 < \log \left( \frac{a}{a - K} \right)$. Therefore, we apply again the local version of \Cref{lemma:differential-exponential-inequality} (see \Cref{rk:local-version-differential-inequality}) to obtain,
    \begin{align*}
        \lim_{s \to t_{1}^-} \renyi{p_s}{p_s'} \leq \log \left( \frac{a}{a - K_n} \right).
    \end{align*}
    Now let us assume that \Cref{eq:recursion-hypothesis} is true for some $k \in \mathds{N}$. By the data processing inequality and the lower semicontinuity of the Rényi divergence, we have that
    \begin{align*}
        \renyi{p_{t_k^-}}{p_{t_k^-}'} \leq \lim_{s \to t_{k}^-} \renyi{p_s}{p_s'} \leq \log \left( \frac{a}{a - K_n} \right). 
    \end{align*}
    Therefore, we can apply \Cref{lemma:differential-exponential-inequality} to obtain that
    \begin{align*}
        \lim_{s \to t_{k+1}^-} \renyi{p_s}{p_s'} \leq \log \left( \frac{a}{a - K_n} \right).
    \end{align*}
    We conclude the proof by lower semicontinuity of the R\'{e}nyi divergence.
    \hfill $\blacksquare$

We end this section by giving the multifractal version of the above theorem. This is an immediate application of the previous proof and the results of \Cref{sec:multifractal-case}.

\begin{restatable}{theorem}{thmDiscreteMultifractal}
    \label{thm:discrete-multifractal}
    Consider the same setting as in \Cref{sec:discrete-time} but replace \Cref{eq:discrete-decursion-pure-jump} by
    \begin{equation}\begin{cases}
          X_{k+1} = \Pi_{\mathcal{C}} \left(X_k - \eta \widehat{g}_S(X_k, \Omega_k) + \sigma_\alpha \eta^{1/\alpha} \xi_k + \sigma_2 \sqrt{2\eta} \zeta_k \right),\\
        X_{k+1}' = \Pi_{\mathcal{C}} \left(X_k' - \eta \widehat{g}_{S'}(X_k', \Omega_k) + \sigma_\alpha \eta^{1/\alpha} \xi_k  + \sigma_2 \sqrt{2\eta} \zeta_k  \right)
\end{cases}\end{equation}
with the same notations and $(\zeta_k)_{k\in\N} \sim \mathcal{N}(0, I_d)^{\otimes \N}$, independent of $(\xi_k)_{k\in\N}$.
    Assume that \Cref{ass:regularity-conditions,ass:finite-sensitivity-on-convex,ass:c2-bounded-assumption} hold and that, for all $t>0$,  $p_t'$ satisfies an $\alpha$-stable Poincaré inequality with constant $(\gamma \sigma^\alpha, \gamma \sigma_2^2)$ for some $\gamma>0$.
    Finally assume that $t\mapsto \renyi{p_t}{p_t'}$ is right-continuous.
    Then,
    \begin{align*}
        \renyi{\law{X_k}}{\law(X_k')} \leq \frac{\beta \sensitivity_{g,\mathcal{C}}^2}{2 \sigma_2^2 n^2} k \eta := K_n k\eta, \quad k\in\N. 
    \end{align*}
    If moreover $K_n < (\gamma \beta)^{-1}$, then,
    \begin{align*}
         \renyi{\law{X_k}}{\law(X_k')} \leq -\log \left( 1 - \frac{\gamma \sensitivity_{g,\mathcal{C}}^2 \beta^2}{2 \sigma_2^2 n^2}  \right), \quad k\in\N.
    \end{align*}
\end{restatable}

\subsection{Omitted proofs of \Cref{sec:stability-properties}}
\label{sec:proofs-stability-properties}

In this section, we use the following classical notation for the variance with respect to a probability measure $\mu \in \mathcal{P}(\Rd)$, for any $\nu$-square-integrable function $f$,
$\mathrm{Var}_\mu (f) := \int f^2 \der \mu - \left( \int f \der \mu \right)^2$.
Equipped with this notation, we can prove the results of \Cref{sec:stability-properties}, starting with the stability properties of the fractional Poincaré inequalities.

\subsubsection{Proof of the stability results for $\alpha$-stable Poincaré inequalities}

\paragraph{Proof of \Cref{lemma:stability-convolution}.}

    Let $f$ be a $\mathcal{C}^1$ function vanishing at infinity. Then, we have
    \begin{align*}
         \text{Var}_{m}(f) &= \intrd\intrd f(x + y)^2 \der\mu(x) \der \mu'(y) - \parenthesis{\intrd\intrd f(x + y) \der\mu(x) \der \mu'(y) }^2 \\
         &= \intrd \left\{ \intrd f(x + y)^2 \der\mu(x) - \parenthesis{\intrd f(x + y) \der\mu(x)}^2 \right\} \der \mu'(y) \\
        &\quad  + \intrd \parenthesis{\intrd f(x + y) \der\mu(x)}^2 \der \mu'(y) - \parenthesis{\intrd\intrd f(x + y) \der \mu(x) \der \mu'(y) }^2 \\
        &=: A_1 + A_2.
     \end{align*}
     
    We can apply the Poincaré's inequality for $\mu$ and Tonelli's theorem to obtain that
    \begin{align*}
        A_1 &\leq \gamma_1 \calphad \intrd \left\{ \intrd\intrdzero \parenthesis{f(x + y) - f(x + y + z)}^2 \frac{\der z}{\normof{z}^{d + \alpha}} \der \mu(x) \right\} \der \mu'(y)  \\& \quad \quad + \gamma_2 \intrd\intrd \normof{\nabla f(x+y)}^2 \der \mu(x) \der \mu'(y) \\
        &= \gamma_1 \calphad  \intrd\intrdzero \parenthesis{f(x) - f(x + z)}^2 \frac{\der z}{\normof{z}^{d + \alpha}} \der m(x)  + \gamma_2 \intrd \normof{\nabla f(x)}^2 \der m(x).
    \end{align*}
    Now, we apply the Poincaré inequality for $\mu'$ on the function $y \mapsto \intrd f(x+y) \der \mu(x)$. This provides the following bound,
    \begin{align*}
        A_2 &\leq \gamma_1' \calphad \intrd\intrdzero \left\{ \intrd f(x + y) \der\mu(x) - \intrd f(x + z + y) \der\mu(x) \right\}^2 \frac{\der z}{\normof{z}^{d + \alpha}} \der \mu' (y)  \\
       & \quad \quad + \gamma_2'\intrd \left\{ \normof{\nabla_y \intrd f(x+y) \der \mu(x)}^2  \right\} \der \mu'(y).
    \end{align*}
    Therefore, by Jensen's inequality,  
    \begin{align*}
        A_2 &\leq \gamma_1' \calphad \intrd\intrdzero\intrd  \left( f(x + y) - f(x + z + y)  \right)^2 \der\mu(x)\frac{\der z}{\normof{z}^{d + \alpha}} \der \mu' (y)  \\
       & \quad   + \gamma_2'\intrd\intrd  \normof{\nabla  f(x+y) }^2 \der \mu(x) \der \mu'(y) \\
        &= \gamma_1' \calphad \intrd\intrdzero  \left( f(x) - f(x + z )  \right)^2 \frac{\der z}{\normof{z}^{d + \alpha}} \der m(x) + \gamma_2'\intrd  \normof{\nabla  f(x) }^2 \der m(x).
    \end{align*}
    This gives the desired result.
    \hfill$\blacksquare$

\paragraph{Proof of \Cref{lemma:bi-lipschitz-stability}.}

Let $f$ be a $\mathcal{C}^1$ function vanishing at infinity. Using the change of variable formula, the Poincaré inequality and Tonelli's theorem, we have 
\begin{align*}
    \text{Var}_{T_{\#}\mu}(f) &= \intrd f^2 \der(T_{\#}\mu) - \parenthesis{\intrd f \der(T_{\#}\mu)}^2= \intrd (f \circ T)^2 d\mu - \parenthesis{\intrd f \circ T d\mu}^2 \\
    &\leq \gamma_1 \calphad \intrd \intrdzero \parenthesis{f(T(x)) - f(T(x + z))}^2 \frac{\der z}{\normof{z}^{d + \alpha}} \der\mu(x) \\ 
    & \quad \quad + \gamma_2 \intrd \normof{\nabla (f \circ T)}^2 (x) \der \mu (x) \\
    &= \gamma_1 \calphad \intrd\int_{\Rd\backslash\{x\}} \parenthesis{f(T(x)) - f(T(y))}^2 \frac{\der y}{\normof{x - y}^{d + \alpha}} \der\mu(x)\\
     & \quad \quad + \gamma_2 \intrd \normof{\mathrm{Jac} (T)(x)}^2_2 \normof{\nabla f (T(x))}^2  \der \mu (x). 
     \end{align*}
     Therefore,
     \begin{align*}
      \text{Var}_{T_{\#}\mu}(f) 
      &\leq \gamma_1 \calphad \intrd\int_{\Rd\backslash\{v\}}\parenthesis{f(u) - f(v)}^2 
      \nonumber\\
      &\qquad\qquad\qquad\qquad\qquad\cdot\frac{ |\det \mathrm{Jac} (T^{-1}) (u) |  |\det \mathrm{Jac} (T^{-1}) (v) |}{\normof{T^{-1}(u) - T^{-1}(v)}^{d + \alpha}} \mu(T^{-1}(u)) \der u \der v\\
    & \quad  + \gamma_2 L_2^2 \intrd \normof{\nabla f (T(x))}^2  \der \mu (x) \\
    &\leq \gamma_1 \calphad L_2^{\alpha + d} \intrd\int_{\Rd\backslash\{v\}}\parenthesis{f(u) - f(v)}^2 \\
    &\qquad\qquad\qquad\qquad\qquad\qquad\cdot\frac{ |\det \mathrm{Jac} (T^{-1}) (u) |  |\det \mathrm{Jac} (T^{-1}) (v) |}{\normof{u - v}^{d + \alpha}} \mu(T^{-1}(u)) \der u \der v \\
     & \quad + \gamma_2 L_2^2 \intrd \normof{\nabla f (T(x))}^2  \der \mu (x) \\
    &= \gamma_1 \calphad L_2^{\alpha + d} \intrd\int_{\Rd\backslash\{T(x)\}}\parenthesis{f(T(x)) - f(v)}^2 \frac{  |\det \mathrm{Jac} (T^{-1}) (v) |}{\normof{T(x) - v}^{d + \alpha}} \mu(x) \der v\der x  \\
     & \quad + \gamma_2 L_2^2 \intrd \normof{\nabla f (T(x))}^2  \der \mu (x) .
\end{align*}
Now we note that $T^{-1}$ is $(1/L_1)$-Lipschitz continuous, so 
that, by Hadamard's inequality \citep{Holland2007}, we have
$\left|\det \mathrm{Jac} (T^{-1}) (v) \right| \leq \prod_{i = 1}^d \normof{\nabla (T^{-1})_i(v)}  \leq  \frac1{L_1^d}$.
Therefore,  
\begin{align*}
      \text{Var}_{T_{\#}\mu}(f) &\leq   \gamma_1 \calphad \frac{L_2^{\alpha + d}}{L_1^d}\intrd\int_{\Rd\backslash\{T(x)\}}\parenthesis{f(T(x)) - f(v)}^2 \frac{ \der v}{\normof{T(x) - v}^{d + \alpha}} \mu(x) \der x \\
       & \quad + \gamma_2 L_2^2 \intrd \normof{\nabla f (T(x))}^2 \der \mu (x) \\
      &\leq  \gamma_1 \calphad \frac{L_2^{\alpha + d}}{L_1^d} \intrd\intrdzero\parenthesis{f(T(x)) - f(T(x) + z)}^2 \frac{\der z}{\normof{z}^{d + \alpha}} \mu(x) \der x \\
       & \quad + \gamma_2 L_2^2 \intrd \normof{\nabla f (T(x))}^2 \der \mu (x),
\end{align*}
which yields the desired result.
\hfill $\blacksquare$

\paragraph{Proof of \Cref{lemma:stability-bounded-perturbation}.}

    Letting $\Phi(x) := x^2$, we have the classical variational formula \citep{bakry_analysis_2014}:
    \begin{align*}
        \mathrm{Var}_\nu(f) = \inf_{a \in \R} \intrd \bregmanphi{f(x)}{a} \der \nu(x).
    \end{align*}
    Therefore:
    \begin{align*}
        &\mathrm{Var}_{\mu'}(f) 
        \\
        &\leq e^b \inf_{a \in \R} \intrd \bregmanphi{f(x)}{a} \der \mu(x) %
        = e^b \mathrm{Var}_\mu(f) \\
        &\leq e^b \left(\gamma_1 C_{\alpha,d}\intrdzero\intrd\frac{(f(x) - f(x+z))^2}{\normof{z}^{d + \alpha}} \der \mu(x) \der z + \gamma_2 \intrd \normof{\nabla f(x)}^2 \der \mu(x) \right) \\
        &\leq e^{2b} \left(\gamma_1 C_{\alpha,d}\intrdzero\intrd\frac{(f(x) - f(x+z))^2}{\normof{z}^{d + \alpha}} \der \mu'(x) \der z + \gamma_2 \intrd \normof{\nabla f(x)}^2 \der \mu'(x) \right), 
    \end{align*}
    which concludes the proof.
    \hfill $\blacksquare$

\paragraph{Proof of \Cref{prop:strongly-convec-poincaré}.} 
    Let us denote $T(x) := x - \eta \nabla \er(x)$. Using the strong convexity and smoothness assumptions, we classically prove that $T: \Rd \to \R$ is bi-Lipschitz with constants given by
    \begin{align*}
        (1 - \eta M) \normof{x - y} \leq \normof{T(x) - T(y)} \leq (1 - \eta \lambda) \normof{x - y}, \quad  x,y\in\Rd.
    \end{align*}
    Now let $X$ be a random variable whose probability distribution satisfies an $\alpha$-stable Poincaré inequality with constant $(c, 0)$ and let $s \in (0, \eta)$. Then, by \Cref{lemma:bi-lipschitz-stability,lemma:stability-convolution}, the distribution of $T(X) + s^{1/\alpha} \sigma L_1^\alpha$ satisfies an $\alpha$-stable Poincaré inequality with constant $(c', 0)$, with
    \begin{align*}
        c':= c \frac{(1 - \eta \lambda)^{\alpha + d}}{(1 - \eta M)^d} + s \sigma^\alpha =: F(\eta)c + s \sigma^\alpha.
    \end{align*}
    We easily see that $F : (0, M^{-1}) \to \R_+$ satisfies $F(0^+) = 1$ and $F(\eta) \to +\infty$ as $\eta \to M^{-1}$. Moreover, 
        $F'(\eta_0) = 0$ if and only if $\eta_0 = \frac{(\alpha + d) \lambda - d M}{\alpha \lambda M}$.
    As $\lambda \leq M$, we have $\eta_0 < 1/M$,  and, on the other hand, \Cref{eq:assumption-strongly-convex-smooth} ensures that $\eta_0 > 0$. Therefore, we necessarily have $0 < F(\eta_0) < 1$. Now we observe that
    $c \leq \frac{\eta \sigma^\alpha}{1 - F(\eta_0)}$ implies $F(\eta) c + \eta \sigma^\alpha \leq \frac{\eta \sigma^\alpha}{1 - F(\eta_0)}$.
    Therefore, we have
    \begin{align*}
        c' \leq F(\eta)c + \eta\sigma^\alpha \leq \frac{\eta \sigma^\alpha}{1 - F(\eta_0)}.
    \end{align*}
    The result follows by an immediate recursion.
    \hfill $\blacksquare$

\vskip 0.2in
\bibliography{main}

\end{document}